\documentclass[letterpaper,11pt]{article}

\usepackage{diagbox}
\usepackage{makecell}
\usepackage{booktabs}
\usepackage{tikz}
\usepackage{amsmath}
\usepackage{amsfonts}
\usepackage{amssymb}
\usepackage[numbers,sort&compress]{natbib}
\usepackage{amsthm}
\usepackage{url,ifthen}
\usepackage{enumitem}
\usepackage{srcltx}
\usepackage{multirow}
\usepackage{boxedminipage}
\usepackage[margin=1.1in]{geometry}
\usepackage{nicefrac}
\usepackage{xspace}
\usepackage{graphicx}
\usepackage{color}
\usepackage{subfigure}
\usepackage{colortbl}
\usepackage{setspace}
\usepackage{mathtools}
\usepackage{algorithm}
\usepackage{algorithmic}

\mathtoolsset{showonlyrefs}

\definecolor{DarkGreen}{rgb}{0.1,0.5,0.1}
\definecolor{DarkRed}{rgb}{0.5,0.1,0.1}
\definecolor{DarkBlue}{rgb}{0.1,0.1,0.5}
\definecolor{Gray}{rgb}{0.2,0.2,0.2}

\newcommand\blfootnote[1]{%
  \begingroup
  \renewcommand\thefootnote{}\footnote{#1}%
  \addtocounter{footnote}{-1}%
  \endgroup
}

\usepackage{listings}
\lstdefinestyle{mystyle}{
    commentstyle=\color{DarkBlue},
    keywordstyle=\color{DarkRed},
    numberstyle=\tiny\color{Gray},
    stringstyle=\color{DarkGreen},
    basicstyle=\footnotesize,
    breakatwhitespace=false,         
    breaklines=true,                 
    captionpos=b,                    
    keepspaces=true,                 
    numbers=left,                    
    numbersep=5pt,                  
    showspaces=false,                
    showstringspaces=false,
    showtabs=false,                  
    tabsize=2
}
\usepackage[small]{caption}
\usepackage{hyperref}
\hypersetup{
    unicode=false,          
    pdftoolbar=true,        
    pdfmenubar=true,        
    pdffitwindow=false,      
    pdfnewwindow=true,      
    colorlinks=true,       
    linkcolor=DarkBlue,          
    citecolor=DarkGreen,        
    filecolor=DarkRed,      
    urlcolor=DarkBlue,          
    pdftitle={},
    pdfauthor={},
}

\def\draft{1}

\def\submit{0}

\ifnum\draft=1 
    \def\ShowAuthNotes{1}
\else
    \def\ShowAuthNotes{0}
\fi

\ifnum\submit=1
\newcommand{\forsubmit}[1]{#1}
\newcommand{\forreals}[1]{}
\else
\newcommand{\forreals}[1]{#1}
\newcommand{\forsubmit}[1]{}
\fi

\ifnum\ShowAuthNotes=1
\newcommand{\authnote}[2]{{ \footnotesize \bf{\color{DarkRed}[#1's Note:
{\color{DarkBlue}#2}]}}}
\else
\newcommand{\authnote}[2]{}
\fi

\usepackage[utf8]{inputenc}
\usepackage[T1]{fontenc}
\usepackage{kpfonts}
\usepackage{microtype}
\newtheorem{theorem}{Theorem}[section]
\newtheorem{remark}[theorem]{Remark}
\newtheorem{lemma}[theorem]{Lemma}
\newtheorem{corollary}[theorem]{Corollary}
\newtheorem{proposition}[theorem]{Proposition}

\theoremstyle{definition}

\newtheorem{example}[theorem]{Example}

\newtheoremstyle{example_contd}
{\topsep} {\topsep}%
{}
{}
{\bfseries}
{.}
{1em}
{\thmname{#1} \thmnumber{ #2}\thmnote{#3} (continued)}

\theoremstyle{example_contd}

\newcommand{\chapterref}[1]{\hyperref[ch:#1]{Chapter~\ref{ch:#1}}}

\newcommand{\claimref}[1]{\hyperref[claim:#1]{Claim~\ref{claim:#1}}}

\newcommand{\corollaryref}[1]{\hyperref[cor:#1]{Corollary~\ref{cor:#1}}}

\newcommand{\definitionref}[1]{\hyperref[def:#1]{Definition~\ref{def:#1}}}

\newcommand{\equationref}[1]{\hyperref[eq:#1]{Equation~\ref{eq:#1}}}

\newcommand{\factref}[1]{\hyperref[fact:#1]{Fact~\ref{fact:#1}}}
\newcommand{\figurelabel}[1]{\label{fig:#1}}
\newcommand{\figureref}[1]{\hyperref[fig:#1]{Figure~\ref{fig:#1}}}

\newcommand{\tableref}[1]{\hyperref[tab:#1]{Table~\ref{tab:#1}}}

\newcommand{\itemref}[1]{\hyperref[item:#1]{Item~(\ref{item:#1})}}
\newcommand{\lemmalabel}[1]{\label{lem:#1}}
\newcommand{\lemmaref}[1]{\hyperref[lem:#1]{Lemma~\ref{lem:#1}}}
\newcommand{\proplabel}[1]{\label{prop:#1}}
\newcommand{\propref}[1]{\hyperref[prop:#1]{Proposition~\ref{prop:#1}}}

\newcommand{\propositionref}[1]{\hyperref[prop:#1]{Proposition~\ref{prop:#1}}}
\newcommand{\remarklabel}[1]{\label{rem:#1}}
\newcommand{\remarkref}[1]{\hyperref[rem:#1]{Remark~\ref{rem:#1}}}
\newcommand{\sectionlabel}[1]{\label{sec:#1}}
\newcommand{\sectionref}[1]{\hyperref[sec:#1]{Section~\ref{sec:#1}}}
\newcommand{\appendixlabel}[1]{\label{app:#1}}
\newcommand{\appendixref}[1]{\hyperref[app:#1]{Appendix~\ref{app:#1}}}
\newcommand{\theoremlabel}[1]{\label{thm:#1}}
\newcommand{\theoremref}[1]{\hyperref[thm:#1]{Theorem~\ref{thm:#1}}}
\newcommand{\examplelabel}[1]{\label{ex:#1}}
\newcommand{\exampleref}[1]{\hyperref[ex:#1]{Example~\ref{ex:#1}}}

\newcommand{\Esymb}{\mathbb{E}}

\DeclareMathOperator*{\E}{\Esymb}

\renewcommand{\hat}{\widehat}

\newcommand{\cB}{{\cal B}}

\newcommand{\cM}{{\cal M}}
\newcommand{\cN}{{\cal N}}

\newcommand{\cS}{{\cal S}}

\newcommand{\defeq}{\stackrel{\small \mathrm{def}}{=}}
\renewcommand{\leq}{\leqslant}

\renewcommand{\geq}{\geqslant}

\newcommand{\paren}[1]{(#1 )}

\newcommand{\set}[1]{\{#1\}}

\newcommand{\abs}[1]{\lvert#1\rvert}

\newcommand{\norm}[1]{\lVert#1\rVert_2}

\newcommand{\R}{\mathbb{R}}

\renewcommand{\D}{\mathcal D}

\usepackage{bm}

\newcommand{\tr}{\mathrm{Tr}}

\newcommand{\ignore}[1]{}

\DeclareMathOperator*{\argmin}{arg\,min}

\renewcommand{\epsilon}{\varepsilon}

\newcommand{\eps}{\epsilon}

\newcommand{\remove}[1]{}

\newcommand{\ploss}[2]{\E_{z \sim \D(#1)}\ell(z; #2)}

\newcommand{\PR}{\mathrm{PR}}
\newcommand{\PRh}{\widehat{\mathrm{PR}}}

\newcommand{\thetaPO}{{\theta_{\mathrm{PO}}}}

\newcommand{\thetaPS}{{\theta_{\mathrm{PS}}}}

\newcommand{\DPR}{{\mathrm{DPR}}}
\newcommand{\thetahat}{{\hat \theta}}
\newcommand{\thetaPOhat}{{\hat{\theta}_{\mathrm{PO}}}}

\newcommand{\zb}{z_0}

\newcommand{\estmu}{\hat{\mu}}
\newcommand{\estD}{\widehat{\D}}
\newcommand{\approxploss}[2]{\E_{z \sim \estD(#1)}\ell(z; #2)}

\def\norm#1{\left\| #1 \right\|}
\def\paren#1{\left( #1 \right)}     
\def\brack#1{\left[ #1 \right]}     

\def\abs#1{\left| #1 \right|}
\newcommand{\grad}{\nabla}

\newcommand{\simiid}{\overset{\textrm{i.i.d.}}{\sim}}
\newcommand{\normal}[2]{\mathcal N \left(#1, #2\right)}

\newcommand{\pr}[1]{\ensuremath{\mathbb{P}} \left\{#1\right\}} 

\title{Outside the Echo Chamber:\\ Optimizing the Performative Risk}

\author{John Miller*~~~~Juan C. Perdomo*~~~Tijana Zrnic*\\
{\small \{miller\_john,
jcperdomo, tijana.zrnic\}@berkeley.edu}
\\ \\University of California, Berkeley}
\begin{document}

\maketitle
\begin{abstract}
In performative prediction, predictions guide decision-making and hence can
influence the distribution of future data. To date, work on performative
prediction has focused on finding performatively stable models, which are the
fixed points of repeated retraining. However, stable solutions can be far from
optimal when evaluated in terms of the performative risk, the loss experienced
by the decision maker when deploying a model. In this paper, we shift attention
beyond performative stability and focus on optimizing the performative risk
directly. We identify a natural set of properties of the loss function and
model-induced distribution shift under which the performative risk is convex, a
property which does not follow from convexity of the loss alone. Furthermore,
we develop algorithms that leverage our structural assumptions to optimize the
performative risk with better sample efficiency than generic methods for
derivative-free convex optimization.
\end{abstract}

\section{Introduction}
\blfootnote{* Equal contribution.}
Predictions in social settings are rarely made in isolation, but rather to
inform decision-making. This link between predictions and decisions causes
predictive models to often be \emph{performative}, meaning they can alter their
environment once deployed. For example, election forecasts impact campaign
spending and affect voter turnout, hence influencing the final election
outcome~\citep{westwood2020projecting}. Similarly, long-term climate forecasts
shape policy decisions which can then affect future weather patterns.

Performative prediction is a recent framework introduced by Perdomo et al. ~\cite{perdomo2020performative} which formalizes the idea that predictive models can impact the data-generating process. So far, work in this area has focused on a particular equilibrium notion known as \emph{performative stability} \cite{drusvyatskiy2020stochastic, mendler2020stochastic, brown2020performative}. Stability is a local definition of optimality, by which a model minimizes the expected risk for the specific distribution that it induces. However, stability provides no general guarantees of performance beyond this equilibrium notion. In fact, stable models can have exceedingly poor \emph{performative risk}, the central measure of performance in the performative prediction framework which captures the true risk incurred by the learner when deploying the model. 

Reasoning by analogy, stable classifiers can be thought of as an \emph{echo chamber} in an online platform. In an echo chamber, one is reassured of their ideas by voicing them, but it's not clear whether they are reasonable outside of this niche community. Similarly, stable classifiers minimize risk on the distribution that they induce, but they provide no global guarantees of performance. 

Therefore, to develop accurate predictions  in performative settings, we  shift attention past performative stability and study optimizing the performative risk directly. This task has so far remained elusive due to the complexities of model-induced distribution shift, i.e. performative effects. In particular, even in simple settings with convex losses, these distribution shifts can make the performative risk non-convex as noted in \cite{perdomo2020performative}. Furthermore, optimizing the performative risk requires a different algorithmic approach than what was previously studied in performative prediction. For instance, the learner needs to actively \emph{anticipate} performative effects rather than myopically retrain until convergence, as the latter would only lead to stability. In short, repeated retraining is an inadequate method of overcoming performative distribution shifts.

\subsection{Our Contributions}

In this paper, we provide the first set of results describing when and how the performative risk may be optimized efficiently. We identify natural assumptions under which the performative risk is convex, even in settings where performative effects can be arbitrarily strong. Furthermore, we study optimization algorithms which explicitly model distribution shift and provably minimize the performative risk in an efficient manner.

To give an overview of our main results, we recall the relevant concepts from the performative prediction framework. Relative to supervised learning, where the learner observes data from a single \emph{static} distribution, the key conceptual innovation in the performative prediction framework is the notion of a \emph{distribution map} $\D(\cdot)$, which maps model parameters $\theta \in \R^d$ to a distribution $\D(\theta)$ over instances $z$. Given a loss $\ell$, the quality of a predictive model parameterized by $\theta$ is measured according to its \emph{performative risk},
\[
\PR(\theta) \defeq \E_{z\sim\D(\theta)} \ell(z;\theta).
\]
 A classifier $\thetaPO$ is \emph{performatively optimal} if it minimizes the performative risk, i.e $\thetaPO \in \argmin_{\theta} \PR(\theta)$. On the other hand, a classifier $\thetaPS$  is \emph{performatively stable} if it satisfies the fixed-point condition, 
 \[
 \thetaPS \in \argmin_{\theta} \ploss{\thetaPS}{\theta}.
 \]
In other words, stable classifiers are those which are optimal for the particular distribution they induce. However, stability has little bearing on whether a classifier has low performative risk. More specifically, the following observation motivates a large part of our later analysis:
\begin{center}
\begin{quote}
\emph{Stable classifiers can maximize the performative risk even when the loss is well-behaved and performative effects are small.}
\end{quote}
\end{center}
Not only can stable points maximize the performative risk, but they can also have an arbitrarily large suboptimality gap, $\PR(\thetaPS) - \PR(\thetaPO)$. The most natural first step towards optimizing the performative risk is to ensure that it is \emph{convex}. Our first main result states that under an appropriate stochastic dominance condition which ensures the distribution map is well-behaved, there exists a critical threshold on the strength of performative effects that guarantees convexity:
\begin{theorem}[Informal]
\theoremlabel{informal1}
Assume that the loss is $\beta$-smooth in $z$ and $\gamma$-strongly convex in $\theta$.  If the  map $\D(\cdot)$ is $\epsilon$-Lipschitz and satisfies an appropriate stochastic dominance condition, then the performative risk is guaranteed to be convex if and only if $\epsilon \leq \frac{\gamma}{2\beta}$.
\end{theorem}
Interestingly, previous work has established that $\epsilon < \gamma / \beta$ is a threshold for repeated retraining to provably converge to a performatively stable point. We show that if we halve this quantity, we get another threshold which determines whether the performative risk is provably convex. 

While \theoremref{informal1} suggests that performative effects need to be small in order to guarantee convexity, we prove that this need not be the case for the setting of \emph{location-scale} families. These are natural classes of distribution maps in which performative effects enter through an additive or multiplicative factor that is linear in $\theta$. Many examples of distribution maps that have appeared in prior work are in fact location-scale families. For this setting, we generalize \theoremref{informal1} to prove the following structural result.
%
\begin{theorem}[Informal]
If the loss is smooth, strongly convex and the map $\D(\cdot)$ is a location-scale family, then the performative risk can be convex irrespective of the Lipschitz constant of $\D(\cdot)$.
\end{theorem}

Finally, having established these structural properties, we turn to algorithms for finding performative optima. Modulo weak regularity assumptions, convexity alone is sufficient to apply classical zeroth-order algorithms in order to find optima in polynomial time. That said, the convergence rate of these algorithms is typically quite slow. 

To address this problem, we propose a \emph{two-stage} approach, by which the learner first creates an explicit model of the distribution map $\hat\D$, and then optimizes a proxy objective for the performative risk obtained by ``plugging in'' $\hat\D$ as if it were really the true distribution map. We instantiate this two-stage procedure in the context of location families, and prove that it optimizes the performative risk with significantly better sample efficiency then generic zeroth-order algorithms.

\subsection{Related Work}

We build on the recent line of work on performative prediction
started by Perdomo et al.~\cite{perdomo2020performative}. While previous papers in this area have focused on performative stability \cite{mendler2020stochastic, drusvyatskiy2020stochastic, brown2020performative}, we move past this solution concept and instead analyze conditions under which one can compute performatively optimal classifiers. 

Given that strategic classification is formally a special case of performative prediction (see \sectionref{experiments} or discussion in \cite{perdomo2020performative} for further details), the study of performative optimality has been implicitly considered in the growing body of work on strategic classification \cite{hardt2016strategic, milli2019social, hu2019disparate, shavit2020learning, bechavod2020causal, chen2020learning, tsirtsis2020decisions, haghtalab2020maximizing}. More specifically, performatively optimal classifiers correspond to Stackelberg equilibria in strategic classification. In contrast to papers within this literature, our analysis relies on identifying macro-level assumptions on the loss and the distribution shift which make the problem tractable, rather than specific micro-level assumptions on the costs or utilities of the agents. For example, Dong et al. \cite{dong2018strategic} prove that the institution's objective (performative risk) is convex by assuming that the agents are rational and compute best-responses according to particular utilities and cost functions. On the other hand, our conditions are on the distribution map and do not directly constrain behavior at the agent level. 

Similarly, several papers in strategic classification  \cite{dong2018strategic,munro2020learning} and policy design  \cite{wager2019experimenting} have recognized that one can apply zeroth-order algorithms \cite{flaxman2005, agarwal2010optimal, shamir2013} to find optima of the institution's risk. The main challenge in applying zeroth-order optimization is the fact that, in general, the performative risk might not satisfy any structural properties which would imply that its stationary points have low risk. One of the main contributions of this paper is precisely to identify under what conditions we can expect this behavior to hold.  

Several works within the economics literature \cite{frankel2019improving, munro2020learning} have also contrasted fixed points of retraining and institutional optima; these analyses resemble our comparisons of stability and optimality, albeit in a more specific setting. Furthermore, there are other settings beyond strategic classification that have similarly studied optimality in the face of performative effects, such as in the context of rankings or selection bias \cite{rosenfeld2020predictions, kilbertus2020fair, tabibian2020design}.


Lastly, our two-stage approach to minimizing the performative risk, whereby we first estimate a model of the distribution map and then optimize a proxy objective, is closely related to ideas in neighboring fields.  At a high level, this general principle has appeared in semiparametric statistics \cite{levit1976efficiency, ibragimov2013statistical, bickel1982adaptive, robinson1988root, newey1990semiparametric} and more recently in double machine learning \cite{chernozhukov2018double, chernozhukov2017double, mackey2018orthogonal}. Furthermore, this idea has been extensively studied in the controls literature where it is referred to as certainty equivalence \cite{theil1957note,simon1956dynamic,mania2019certainty,simchowitz2020naive}, or as model-based planning in reinforcement learning \cite{agarwal_planning}.

\subsection{Additional Preliminaries}
\sectionlabel{preliminaries}

As done by previous works in this area, we limit ourselves to considering predictive models parameterized by a finite-dimensional vector $\theta \in \Theta \subseteq \R^d$, where $\Theta$ is a closed, convex set. The distribution map $\D(\cdot)$ maps parameter vectors to data distributions over real-valued instances $z \in~\R^m$. While each model $\theta$ can induce a potentially distinct distribution $\D(\theta)$, we expect similar classifiers to induce similar distributions. This intuition is captured by the notion of $\epsilon$-\emph{sensitivity}, which is essentially a Lipschitz condition on the distribution map $\D(\cdot)$. We state that $\D(\cdot)$ is $\epsilon$-sensitive for some $\epsilon \geq 0$ if for all $\theta, \theta' \in \Theta$,
\begin{align}
	W_1(\D(\theta), \D(\theta')) \leq \epsilon \norm{\theta - \theta'}_2.
	\tag{A1}
\label{ass:sensitivity}
\end{align}
Here, $W_1$ denotes the Wasserstein-1 or earth mover's distance between two distributions.

\section{Contrasting Optimality and Stability}
\label{sec:stability_vs_optimality}

Up until now, all works within the performative prediction literature have focused on analyzing when different algorithms converge to stable points. While the primary motivation for stability was eliminating the need for retraining, it was observed as a useful byproduct that stable points can approximately minimize the performative risk. 

More specifically, Perdomo et al. ~\cite{perdomo2020performative} prove that all stable points and performative optima lie within $\ell_2$-distance at most $2 L_z \epsilon /\gamma$ of each other, where $\epsilon$ is the sensitivity of the distribution map, $\gamma$ denotes the strong convexity parameter of the loss, and $L_z$ denotes the Lipschitz constant of the loss \emph{in} $z$. At first glance, this result implicitly suggests that stable points also have good predictive performance. While this is sometimes the case, in many settings $L_z$ is large enough to make the bound vacuous. For example, there exist cases where the performative risk is strongly convex, but stable points actually \emph{maximize} the performative risk.

\begin{proposition}
\proplabel{quadratic}
For any $\gamma,\Delta> 0$, there exists a performative prediction problem where the loss is $\gamma$-strongly convex in $\theta$, yet the unique stable point $\thetaPS$ maximizes the performative risk and $\PR(\thetaPS)-\min_\theta\PR(\theta) \geq \Delta$.
\end{proposition}

\begin{proof}
We prove the proposition by constructing an example. Let $z \sim \D(\theta)$ be a point mass at $\epsilon \theta$, and define the loss to be:
	\[
	\ell(z;\theta) = -\beta\cdot \theta^\top z + \frac{\gamma}{2} \|\theta\|_2^2,
	\]
for some $\beta\geq 0$. This loss is $\gamma$-strongly convex and the distribution map is $\epsilon$-sensitive.
A short calculation shows that the performative risk simplifies to
\begin{equation}
\label{eqn:pr-quadratic}
\PR(\theta) = \left(\frac{\gamma}{2} - \epsilon \beta\right) \cdot \|\theta\|_2^2.
\end{equation}
For $\epsilon\neq \gamma / \beta$, there is a unique performatively stable point at the origin, and if $\epsilon > \tfrac{\gamma}{2\beta}$ this point is the unique maximizer of the performative risk. Moreover, for $\epsilon > \tfrac{\gamma}{2\beta}$, $\min_\theta\PR(\theta) = (\gamma/2 - \epsilon \beta) \cdot \max_{\theta\in\Theta}\|\theta\|_2^2$. Therefore, depending on the radius of $\Theta$, the suboptimality gap of $\thetaPS$ can be arbitrarily large.
\end{proof}

In the above example, $\nabla_\theta \ell(z;\theta)$ is $\beta$-Lipschitz in $z$, a condition commonly referred to as \emph{smoothness} in prior work on performativity. The previous proposition thus shows that stable points can have an arbitrary suboptimality gap when $\epsilon>\frac{\gamma}{2\beta}$. This is important since $\epsilon<\frac{\gamma}{\beta}$ is the regime where previously studied algorithms for optimizing under performativity---such as repeated risk minimization or different variants of gradient descent \cite{perdomo2020performative, mendler2020stochastic}---converge to stability. Applying these methods when $\epsilon \in (\gamma / (2\beta), \gamma/\beta)$ would hence maximize the performative risk on this problem.

Moreover, we remark that the Lipschitz constant $L_z$  is equal to $\beta\cdot \max_{\theta\in\Theta}\|\theta\|_2$. Therefore, the results of \cite{perdomo2020performative} imply that stable points and optima are at distance at most $\frac{2L_z\epsilon}{\gamma} = \frac{2\beta\epsilon}{\gamma}\max_{\theta\in\Theta}\|\theta\|_2$. When $\epsilon > \frac{\gamma}{2\beta}$, as assumed in the proof of \propref{quadratic}, this bound on the distance becomes vacuous: $\|\thetaPS - \thetaPO\|_2\leq \max_{\theta\in\Theta}\|\theta\|_2$.
	
Lastly, we  point out that $\epsilon = \frac{\gamma}{2\beta}$ is a sharp threshold for convexity of the performative risk in this example, as can be seen in Equation \eqref{eqn:pr-quadratic}. In the following section, we show that this threshold behavior is not an artifact of this particular setting, but rather a phenomenon that holds more generally.

\section{Convexity of the Performative Risk}
\sectionlabel{structural_results}

We now introduce our main structural results illustrating how the performative risk can be convex in various natural settings, and hence amenable to direct optimization. Throughout our presentation, we adopt the following convention. We state that the performative risk is $\lambda$-convex, for some $\lambda\in\R$, if the objective,
\[
\PR(\theta) - \frac{\lambda}{2}\|\theta\|_2^2
\]
is convex. In other words, if $\lambda$ is positive, then $\PR(\theta)$ is $\lambda$-strongly convex. If $\lambda$ is negative, then adding the analogous  
regularizer $\tfrac{\lambda}{2} \norm{\theta}^2$ ensures $\PR(\theta)$ is convex. Furthermore, in addition to $\epsilon$-sensitivity, we will make repeated use of the following assumptions throughout the remainder of the paper. To facilitate readability, we let $\mathcal{Z} \defeq \cup_{\theta\in\Theta} \text{supp}(\D(\theta))$. We say that a loss function $\ell(z;\theta)$ is $\beta$-\emph{smooth} in $z$ if for all $\theta \in\Theta$ and $z, z' \in \mathcal{Z}$,
\begin{equation}
	\norm{\nabla_\theta\ell(z;\theta) - \nabla_\theta \ell(z';\theta)}_2 \leq \beta \norm{z - z'}_2.
\tag{A2}
\label{ass:smoothness_z}
\end{equation}
Furthermore, a loss function $\ell(z;\theta)$ is $\gamma$-\emph{strongly convex} in $\theta$ if for all $\theta, \theta',\theta_0 \in \Theta$,
\begin{align}
\E_{z\sim\D(\theta_0)} \ell(z;\theta) \geq \E_{z\sim\D(\theta_0)}\ell(z;\theta') \nonumber
 \quad + \E_{z\sim\D(\theta_0)}\nabla_\theta \ell(z;\theta')^\top (\theta-\theta') + \frac{\gamma}{2}\left\|\theta-\theta'\right\|_2^2.
\tag{A3a}
\label{ass:s_convexity_theta}
\end{align}
If $\gamma=0$, this assumption is equivalent to convexity. Similarly, we say that the loss is $\gamma_z$-strongly convex in $z$ if for all $\theta \in\Theta$ and $z, z' \in \mathcal{Z}$,
\begin{equation}
\ell(z;\theta) \geq \ell(z';\theta) + \nabla_{z} \ell(z';\theta)^\top (z' - z) + \frac{\gamma_z}{2}\left\|z - z'\right\|_2^2.
\tag{A3b}
\label{ass:s_convexity_z}
\end{equation}
Lastly, we state that a distribution map, loss pair $\left(\D(\cdot), \ell\right)$ satisfies \emph{mixture dominance} if the following condition holds for all $\theta, \theta',\theta_0 \in \Theta$ and $\alpha \in (0,1)$:
\begin{equation}
\ploss{\alpha \theta + (1-\alpha) \theta'}{\theta_0} \leq \E_{z \sim \alpha\D(\theta) + (1-\alpha)\D(\theta')} \ell(z; \theta_0).
\tag{A4}
\label{ass:mixture}
\end{equation}

Smoothness and strong convexity are standard and have appeared previously in the context of performative prediction. The mixture dominance condition is novel and plays a central role in our analysis of when the performative risk is convex. To provide some intuition for this condition, we recall the definition of the \emph{decoupled performative risk}:
\begin{equation*}
	\DPR(\theta,\theta') = \E_{z\sim\D(\theta)}\ell(z;\theta').
\end{equation*}
Notice that asserting convexity of the performative risk is equivalent to showing convexity of $\DPR(\theta,\theta)$ when both arguments are forced to be the same. While convexity \eqref{ass:s_convexity_theta} guarantees that $\DPR$ is convex in the second argument, mixture dominance \eqref{ass:mixture} essentially posits convexity of $\DPR$ in the first argument. Importantly, assuming convexity in each argument separately does \emph{not} directly imply that the performative risk is convex. 

On a more intuitive level, this assumption \eqref{ass:mixture} is essentially a stochastic dominance statement: the mixture distribution $\alpha\D(\theta)+(1-\alpha)\D(\theta')$ ``dominates'' $\D(\alpha\theta + (1-\alpha)\theta')$ under a certain loss function. Similar conditions have been extensively studied within the literature on stochastic orders \cite{shaked2007stochastic}, which we further discuss in \appendixref{stochastic_orders}. Part of our analysis relies on incorporating tools from this literature, and we believe that further exploring technical connections between this field and performative prediction could be valuable. For example, using results from stochastic orders we can show that \eqref{ass:mixture} holds when the loss is convex in $z$ and the distribution map $\D(\cdot)$ forms a \emph{location-scale family} of the form: \begin{equation}
\label{eqn:scale-loc-def}
	z_\theta\sim\D(\theta)~ \Leftrightarrow~ z_\theta\stackrel{d}{=} (\Sigma_0 + \Sigma(\theta)) \zb + \mu_0 + \mu \theta,
\end{equation}
where $\zb\sim\D_0$ is a sample from a fixed zero-mean distribution $\D_0$, and $\Sigma(\theta),\mu$ are linear maps  (see \propositionref{sl-weak-convexity} for a formal proof). Distribution maps of this sort are ubiquitous throughout the performative prediction literature and hence satisfy mixture dominance if the loss $\ell$ is convex. For instance, the distribution map for the strategic classification simulator in \cite{perdomo2020performative} is a location family. Other examples of location families can be found in previous work on strategic classification \cite{frankel2019improving, haghtalab2020maximizing}. Mixture dominance can also hold in discrete settings, e.g. $\D(\theta) = \text{Bernoulli}(a^\top \theta + b)$ satisfies this condition for any loss. Having provided some context on the mixture dominance condition, we can now state the main result of this section:
\begin{theorem}
\theoremlabel{convex-pr}
Suppose that the loss function $\ell(z;\theta)$ is $\gamma$-strongly convex in $\theta$ \eqref{ass:s_convexity_theta}, $\beta$-smooth in $z$ \eqref{ass:smoothness_z}, and that $\D(\cdot)$ is $\epsilon$-sensitive \eqref{ass:sensitivity}. If mixture dominance \eqref{ass:mixture} holds, then the performative risk is $\lambda$-convex for $\lambda=\gamma - 2\epsilon\beta$.
\end{theorem}
Together with the example from the proof of \propref{quadratic}, this theorem shows that $\frac{\gamma}{2 \beta }$ is a sharp threshold for convexity of the performative risk. If $\epsilon$ is strictly less than this threshold, then under mixture dominance and appropriate conditions on the loss, the performative risk is strongly convex by \theoremref{convex-pr}. On the other hand, if $\epsilon$ is above this threshold, the example from \propref{quadratic} shows that there exists a performative prediction instance which satisfies the remaining assumptions, yet is non-convex; in particular, for $\epsilon > \frac{\gamma}{2\beta}$ the performative risk is strictly concave in that example. This threshold was also implicitly observed by Perdomo et al.~\cite{perdomo2020performative} in the proof of Proposition 4.2 as byproduct of showing that the performative risk can be non-convex for $\epsilon \leq \frac{\gamma}{\beta}$. However, they provide no general analysis of when the performative risk is convex. Note that all of the above examples satisfy mixture dominance.



While the threshold $\epsilon = \gamma / (2\beta)$ is in general tight as argued above, for certain families of distribution maps the conclusion of \theoremref{convex-pr} can be made considerably stronger. Indeed, in some cases the performative risk is convex \emph{regardless} of the magnitude of performative effects, as observed for the following location family.

\begin{example}
\examplelabel{linear_reg}
Consider the following stylized model of predicting the final vote margin in an election contest. Features $x$, such as past polling averages, are drawn i.i.d. from a static distribution,
$x\sim \D_x$. Since predicting a large margin in either direction can dissuade people from voting, we consider outcomes drawn from the conditional distribution:  $y|x\sim g(x) + \mu^\top \theta + \xi$, where $g:\R^d\rightarrow\R$ is an arbitrary map, $\mu\in\R^d$ is a fixed vector, and $\xi$ is a zero-mean noise variable. If $\ell$ is the squared loss, $\ell((x,y);\theta) = \frac{1}{2}(y - x^\top \theta)^2$, or the absolute loss, $\ell((x,y);\theta) = |y - x^\top \theta|$, then the performative risk is convex for any $g$ and $\mu$.
\end{example}

The proof follows by simply observing that in both cases, the performative risk  can be written as a linear function in $\theta$ composed with a convex function. Another interesting property of this example is that the distribution map is $\epsilon$-sensitive with $\epsilon = \norm{\mu}_2$, yet the sensitivity parameter plays no role in the characterization of convexity. Motivated by this observation, we specialize the analysis in \theoremref{convex-pr} to the particular case of location-scale families, and obtain a result that is at least as tight as the previous theorem.
\begin{theorem}
\theoremlabel{convexity-location-fam}
Suppose that $\ell(z;\theta)$ is $\gamma$-strongly convex in $\theta$ \eqref{ass:s_convexity_theta}, $\beta$-smooth \eqref{ass:smoothness_z}, and $\gamma_z$-strongly convex in $z$ \eqref{ass:s_convexity_z}. Furthermore, suppose that $\D(\theta)$ forms a location-scale family \eqref{eqn:scale-loc-def} with $\epsilon$ as its sensitivity parameter\footnote{The sensitivity parameter $\epsilon$ for location-scale families can be explicitly bounded in terms of the parameters $\mu$ and $\Sigma(\theta)$; see \remarkref{eps_bound} in the Appendix.}. Define  $\Sigma_{\zb}$ to be the covariance matrix of $\zb\sim\D_0$, and let
\begin{align*}
\sigma_{\min}(\mu)=\min_{\|\theta\|_2=1}\|\mu\theta\|_2, \sigma_{\min}(\Sigma) &=\min_{\|\theta\|_2=1}\|\Sigma_{\zb}^{1/2}\Sigma(\theta)^\top\|_F.
\end{align*}
 Then, the performative risk is $\lambda$-convex for $\lambda$ equal to:
 \[ \max\{\gamma - \beta^2/\gamma_z, \gamma - 2\epsilon\beta + \gamma_z(\sigma_{\min}^2(\mu) + \sigma_{\min}^2(\Sigma))\}.\]
\end{theorem}

This tighter bound leverages the fact that some losses are strongly convex in the performative variables, such as the squared loss when only the outcome variable exhibits performative effects. In general, one can achieve a tighter analysis of when the performative risk is convex by distinguishing between variables which are \emph{static}, whose distribution is the same under $\D(\theta)$ for all $\theta$, and performative variables which are influenced by the deployed classifier. For the most part we avoid this distinction in the main body for the sake of readability, however, we elaborate on how the analysis can be strengthened in \appendixref{static_vs_perf}. We now illustrate an application of \theoremref{convexity-location-fam} on a scale family example.

\begin{example}
Suppose that $x>0$ is a one-dimensional feature drawn from a fixed distribution $\D_x$, and let $y|x\sim \theta x\cdot\text{Exp}(1)$ be distributed as an exponential random variable with mean $\theta x$. Let the loss be the squared loss, $\ell((x,y);\theta) = \frac{1}{2}(y-\theta\cdot x)^2$ and let $\Theta = \R_{+}$. Note that this example exhibits a self-fulfilling prophecy property whereby all solutions are performatively stable. On the other hand, $\PR(\theta) = \theta^2 \E x^2$, and the unique performative optimum is $\thetaPO = 0$. Again, we see how stability has no bearing on whether a solution has low performative risk.

However, we note that the loss is 1-strongly convex in $y$. Furthermore, by averaging over the static features, we observe that $\PR(\theta)$ is $\E x^2$-strongly convex in $\theta$ and $\E x$-smooth in $y$. Therefore, according to \theoremref{convexity-location-fam}, the performative risk is convex and hence tractable to optimize, since $\gamma - \beta^2/\gamma_z = \E x^2-(\E x)^2\geq 0$ by Jensen's inequality.
\end{example}

While this example, like most others in this section, is intended as a toy problem to provide the reader with some intuition regarding the intricacies of performativity, many instances of performative prediction in the real world do exhibit a self-fulfilling prophecy aspect whereby predicting a particular outcome increases the likelihood that it occurs. For instance, predicting that a student is unlikely to do well on a standardized exam may discourage them from studying in the first place and hence lower their final grade. Settings like these where stability is a vacuous guarantee of performance remind us how developing reliable predictive models requires going outside the stability echo chamber.  

As a final note, to prove the results in this section, we have imposed additional assumptions such as mixture dominance, or analyzed the special case of location-scale families. The reader might naturally ask whether these settings are so restrictive that one can optimize the performative risk using previous optimization methods for performative prediction which find stable points. Or in particular, whether stable points and performative optima now identify. 

It turns out that both solutions can still have qualitatively different behavior, regardless of the strength of performative effects. First, notice that the example in the proof of \propref{quadratic} is a location family, and as such it satisfies mixture dominance. In that example, when $\epsilon\in(\frac{\gamma}{2\beta},\frac{\gamma}{\beta})$, methods for finding stable points converge to a maximizer of the performative risk; however, this is outside the regime where the performative risk is convex. In what follows, by relying on \theoremref{convexity-location-fam}, we provide another scale family example where the performative risk is convex regardless of $\epsilon$, yet stable points can be arbitrarily suboptimal.

\begin{example}
Suppose that $\D(\theta) = \mathcal{N}(\mu,\epsilon^2\theta^2)$ for some $\mu\in\R$ and $\epsilon> 0$. This distribution map is $\epsilon$-sensitive. Furthermore, if $\ell$ is the squared loss, $\ell(z;\theta) = \frac{1}{2}(z-\theta)^2$, then there is a unique stable point $\thetaPS = \mu$. On the other hand, $\thetaPO = \mu/(1+\epsilon^2)$. 

Notice how, contrary to the performative optimum $\thetaPO$,  the stable point $\thetaPS$ is independent of $\epsilon$ and hence oblivious to the performative effects. Depending on $\mu$, the stable point can be arbitrarily suboptimal, since $\PR(\thetaPS) - \PR(\thetaPO)=\Omega(\mu^2)$. Note also that, according to \theoremref{convexity-location-fam}, the performative risk is $\gamma - 2\epsilon\beta + \gamma_z \sigma_{\min}^2(\Sigma) = 1-2\epsilon + \epsilon^2$-convex. Since $1-2\epsilon + \epsilon^2 = (\epsilon-1)^2\geq 0$, the performative risk is always convex and hence tractable to optimize.
\end{example}

\section{Optimization Algorithms}
\sectionlabel{algorithms}

Having identified conditions under which the performative risk is convex, we now consider methods for efficiently optimizing it. One of the main
challenges of carrying out this task is that, even in convex settings, the learner can only access the objective via noisy function
evaluations corresponding to classifier deployments. Without
knowledge of the underlying distribution map, it is infeasible to compute gradients of
the performative risk. A naive solution is to apply a zeroth-order method, however, these algorithms are in general hard to tune, and their performance scales poorly with the problem dimension. 

Our main algorithmic contribution is to show how one can address these issues by creating an explicit \emph{model} of the distribution map and then optimizing a proxy objective for the performative risk offline. We refer to this as the two-stage procedure for optimizing the performative risk and show it is provably efficient for the case of location families.   

To develop further intuition, consider the following simple example. Let $z \sim \mathcal{N}(\epsilon \theta,  1)$ be a one-dimensional Gaussian and let $\ell(z;\theta) = \frac{1}{2}(z - \theta)^2$ be the squared loss. Then, the performative risk, $\PR(\theta) = \tfrac{1}{2}(\epsilon - 1)^2 \theta^2$, is a simple, convex function for all values of $\epsilon$ (as indeed confirmed by \theoremref{convexity-location-fam}, since $\gamma - 2\epsilon\beta + \gamma_z \sigma_{\min}^2(\mu) = 1-2\epsilon + \epsilon^2\geq 0$). However, gradients are unavailable since they depend on the density of $\D(\theta)$, denoted $p_\theta$, which is typically unknown:
\begin{align*}
\nabla_\theta \PR(\theta) &= \E_{z \sim \D(\theta)} \nabla_\theta \ell(z;\theta) 
    + \E_{z \sim \D(\theta)}\ell(z;\theta) \grad_\theta \log p_\theta(z)\\
    &= \E_{z \sim \D(\theta)} - (z - \theta) + \epsilon (\epsilon -1 )\theta. 
\end{align*}
Despite the simplicity of this example, earlier approaches to optimization in performative prediction, such as repeated retraining \cite{perdomo2020performative}, fail on this problem. The reason is that they essentially ignore the second term in the gradient computation which requires explicitly anticipating performative effects. For example, retraining computes the sequence of updates $\theta_{t+1} = \argmin_\theta \E_{z\sim \D(\theta_t)} \tfrac{1}{2}(z - \theta)^2 = \epsilon \theta_t$, which diverges for $|\epsilon| > 1$. 
\subsection{Generic Derivative-Free Methods}
\sectionlabel{derivative_free_methods}
Having observed the difficulty of computing gradients, the most natural starting
point for optimizing the performative risk is to consider derivative-free
methods for convex optimization~\cite{flaxman2005, agarwal2010optimal,shamir2013}. These methods work by constructing a noisy estimate of the
gradient by querying the objective function at a randomly perturbed point around
the current iterate. For instance, Flaxman et al.~\cite{flaxman2005} sample a vector $u \sim
\text{Unif}(\cS^{d-1})$ to get a slightly biased gradient estimator,
\begin{align*}
    \grad_\theta \PR(\theta) \approx \frac{d}{\delta} \E[\PR(\theta + \delta u)u],
\end{align*}
for some small $\delta>0$. Generic derivative-free algorithms for convex optimization require few
assumptions beyond those given in the previous section
to ensure convexity. Moreover, they guarantee convergence to a performative
optimum given sufficiently many samples.  However, their rate of convergence
can be slow and scales poorly with the problem dimension. In general,
zeroth-order methods require $\widetilde{O}(d^2 / \Delta^2)$ samples to obtain a
$\Delta$-suboptimal point~\citep{agarwal2010optimal,shamir2013}, which can be
prohibitively expensive if samples are hard to come by.

\subsection{Two-Stage Approach}
In cases where we have further structure, an alternative solution to derivative-free methods is to utilize a \emph{two-stage}
approach to optimizing the performative risk. In the first stage,
we estimate a coarse model of the distribution map, $\estD(\cdot)$ via experiment design. Then, in
the second stage, the algorithm optimizes a proxy to the performative risk treating the estimated $\estD$ as if it were the true distribution map:
\begin{align*}
    \thetaPOhat
    \in \argmin_\theta \PRh(\theta)
    \defeq \approxploss{\theta}{\theta}.
\end{align*}
The exact implementation of this idea depends on the problem setting at hand; to make things concrete, we instantiate the approach in the context of location families and prove that it optimizes the performative risk with significantly better sample
complexity than generic zeroth-order methods. For the remainder of this section, we assume the distribution map $\D$ is parameterized by a location family
\begin{align*}
    z_\theta \sim \D(\theta)
    ~ \Leftrightarrow~ 
    z_\theta\stackrel{d}{=} \zb + \mu\theta,
\end{align*} 
where the matrix $\mu \in \R^{m \times d}$ is an unknown parameter, and $\zb\sim\D_0$ is
a zero-mean random variable.\footnote{The variable $z_0$ being zero-mean is only to simplify the exposition; the same analysis carries over when there is an additional intercept term. Similarly, the choice of Gaussian noise in the experiment design phase of Algorithm \ref{alg:model_based_location_fam} is made for convenience. In general, any subgaussian distribution with full rank covariance would suffice.}

As discussed previously, location-scale families encompass many formal
examples discussed in prior work. They capture the intuition that in performative
settings, the data points are composed of a \emph{base} component $z_0$,
representing the natural data distribution in the absence of performativity, and an additive performative term.  


\begin{algorithm}[tb]
   \caption{Two-Stage Algorithm for Location Families}
   \label{alg:model_based_location_fam}
\begin{algorithmic}
    \STATE \textbf{Stage 1:} Construct a model of the distribution map
    \STATE // Estimate location parameter $\mu$ with experiment design
    \FOR{$i=1$ {\bfseries to} $n$}
        \STATE -Sample and deploy classifier $\theta_i \simiid \normal{0}{I_d}$.
        \STATE -Observe $z_i \sim \D(\theta_i)$.
    \ENDFOR
    \STATE -Estimate $\mu$ via ordinary least squares, 
    $\estmu \in \argmin_{\mu} \sum_{i=1}^n \norm{z_i - \mu \theta_i}_2^2$.
    \STATE // Gather samples from the base distribution
    \FOR{$j=n+1$ {\bfseries to} $2n$}
        \STATE -Deploy classifier $\theta_j = 0$, and observe $z_j \sim \D(0)$.
    \ENDFOR
    \STATE \textbf{Stage 2:} Minimize a finite-sample approximation of the performative risk,
    $\argmin_{\theta \in \Theta} \frac{1}{n}\sum_{j=n+1}^{2n} \ell(z_j +
    \estmu\theta; \theta)$.
\end{algorithmic}
\end{algorithm}

In the first stage of our two-stage procedure we build a model of the distribution map $\estD$ that in effect 
allows us to draw samples $z \sim \estD(\theta) \approx \D(\theta)$. To do this, we perform
experiment design to recover the unknown parameter $\mu$ which captures the performative effects. In particular, we sample and deploy $n$ classifiers
$\theta_i$, $i\in[n]$, observe data $z_i \sim
\D(\theta_i)$, and then construct an estimate $\estmu$ of the location map
$\mu$ using ordinary least squares.  We then gather samples from the base distribution
$\D_0$ by repeatedly deploying the zero classifier.  In the
location-family model, deploying the zero classifier ensures we observe data
points $\zb$, without performative effects. With both of these components, given
any $\theta'$, we can simulate $z \sim \estD(\theta')$ by taking $z =
\zb + \estmu\theta'$. 

In the second stage, we use the estimated model to construct a proxy objective. Define the perturbed performative risk:
\begin{align*}
    \PRh(\theta)
    = \approxploss{\theta}{\theta}
    = \E_{\zb \sim \D_0}\ell(\zb + \estmu\theta; \theta).
\end{align*}
Note that $\PR(\theta)= \E_{\zb \sim \D_0}\ell(\zb + \mu\theta; \theta)$.
Using the estimated parameter $\estmu$ and samples $z_i \sim \D_0$, we
can construct a finite-sample approximation to the perturbed performative risk and find the following optimizer:
\begin{align*}
    \thetahat_n 
    \in \argmin_{\theta \in \Theta} \PRh_n(\theta)
    &\defeq \frac{1}{n}\sum_{i=n+1}^{2n} \ell(z_i + \estmu\theta; \theta).
\end{align*}
The main technical result in this section shows that, under appropriate regularity assumptions on the loss,
Algorithm~\ref{alg:model_based_location_fam} efficiently approximates the performative optimum. In particular, when the data dimensionality $m$ is comparable to the model dimensionality $d$, i.e. $m = O(d)$, then computing a $\Delta$-suboptimal classifier requires $O(d / \Delta)$ samples. In contrast, the derivative-free methods considered previously require $\widetilde{O}(d^2/\Delta^2)$ samples to compute a classifier of similar quality. The formal statement and proof of this result is deferred to~\appendixref{model_based_analysis}. 
\begin{theorem}[Informal]
Under appropriate smoothness and strong convexity assumptions on the loss
$\ell$, if the distribution of $\zb$ is subgaussian, and if 
the number of samples $n \geq \Omega\paren{d + m + \log(1/\delta)}$, then,
with probability $1-\delta$, Algorithm~\ref{alg:model_based_location_fam}
returns a point $\thetahat_n$ such that
\begin{align*}
    \PR(\thetahat_n) - \PR(\thetaPO)
    \leq \mathcal{O}\paren{\frac{d + m + \log(1/\delta)}{n} 
                    + \frac{1}{\delta n}}.
\end{align*}
\end{theorem}

While we analyze this two-stage procedure in the context of location families,
the principles behind the approach can be extended to more general
settings. Whenever the distribution map has enough structure to efficiently
estimate a model $\estD$ that supports sampling new data, we
can always use the ``plug-in'' approach above and construct and optimize a
perturbed version of the performative risk.

\section{Experiments}
\sectionlabel{experiments}
We complement our theoretical findings with an empirical evaluation of different methods on two tasks: the strategic classification simulator from \cite{perdomo2020performative}, and a synthetic linear regression example.

We pay particular attention to understanding the
differences in empirical performance between algorithms which converge to
performative optima, such as the two-stage procedure or derivative-free methods
from \sectionref{derivative_free_methods}, versus existing optimization
algorithms for finding stable
points, in particular greedy and lazy SGD due to Mendler-D\"{u}nner et al.\cite{mendler2020stochastic}. In addition, we focus on highlighting the differences in the sample
efficiency of the different algorithms and examine their sensitivity to the
relevant structural assumptions outlined in \sectionref{structural_results}. To evaluate derivative-free methods, we implement the ``gradient descent without a gradient'' algorithm from~\cite{flaxman2005}, which we refer to from here on out as the  ``DFO algorithm.'' For each of the following experiments, we run each algorithm 50 times and 
    display 95\% bootstrap confidence intervals.
 We provide a formal description of
all the procedures, as well as a detailed description of the experimental setup in \appendixref{experiments}.

%
%

\textbf{Linear regression experiments.}
We begin by evaluating how increasing the strength of performative effects affects the behavior of the different optimization procedures in settings where the performative risk is convex. We recall the setup from \exampleref{linear_reg}, where the learner attempts to solve a linear regression with performative labels. Given a parameter $\theta$, data are drawn from $\D(\theta)$ according to:
\begin{align*}
	x \sim \mathcal{N}(0, \Sigma_x),~ U_y \sim \mathcal{N}(0, \sigma_y^2),~ y = \beta^\top x + \mu^\top \theta + U_y.
\end{align*}
This distribution map is a location family, and is $\epsilon$-sensitive with $\epsilon= \norm{\mu}_2$. Performance is measured according to the squared loss, $\ell((x,y);\theta) = \tfrac{1}{2}(y - \theta^\top x)^2$.  Furthermore, the performative risk is convex for all choices of $\mu$. 
\begin{figure}
\begin{center}
    \includegraphics[width=\linewidth]{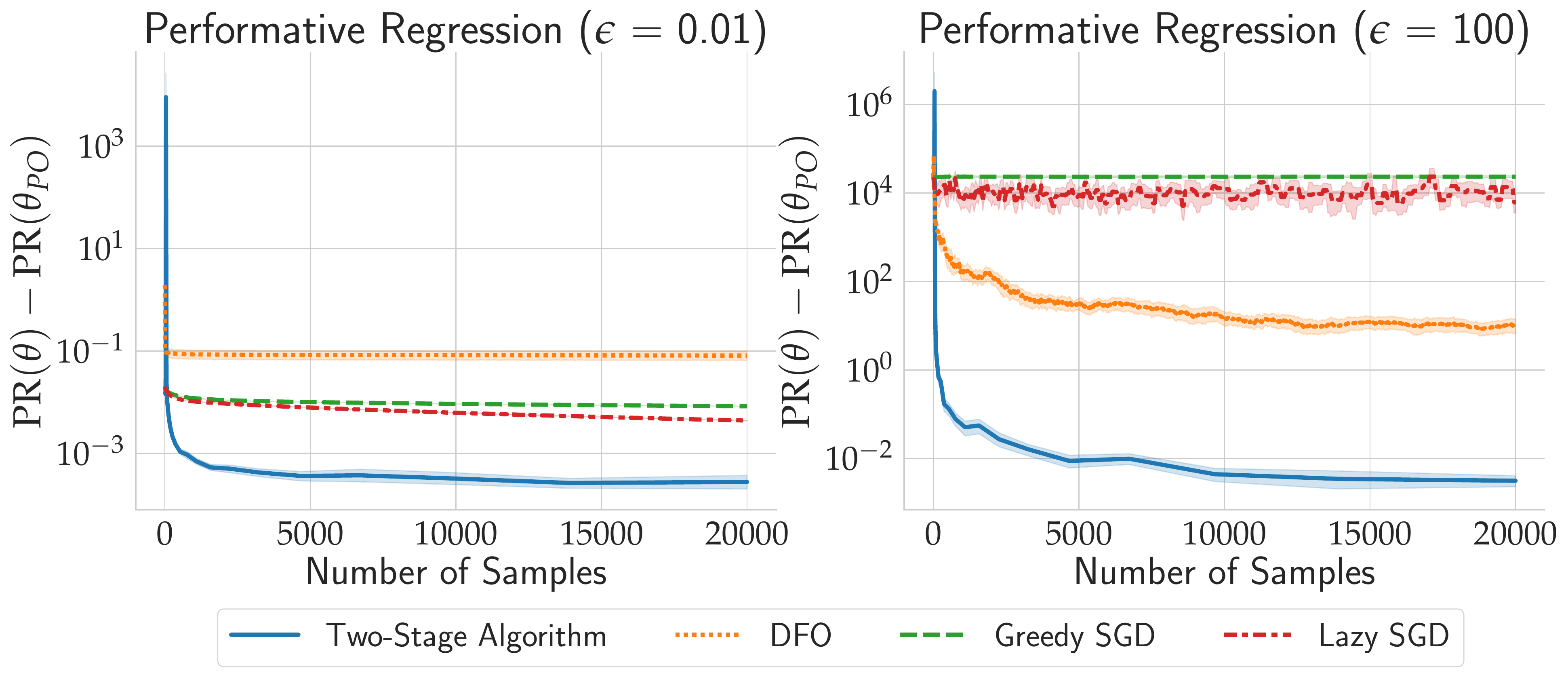}
    \caption{Suboptimality gap versus number of samples collected for the two-stage
    algorithm, DFO algorithm, greedy SGD, and lazy SGD, for $\epsilon = 0.01$
    (left) and $\epsilon = 100$ (right). Each experiment is repeated 50
    times, and we display 95\% bootstrap confidence intervals.}
   \figurelabel{linear_regression_experiment}
\end{center}
\end{figure}

For small $\epsilon$, we see that greedy and lazy SGD converge to a stable point that approximately minimizes the performative risk (see left panel in \figureref{linear_regression_experiment}). However, as we increase the strength of performative effects, these methods fail to make any progress, and are outperformed by both the DFO algorithm and the two-stage approach by a considerable margin (see right panel in \figureref{linear_regression_experiment}). The two-stage procedure efficiently converges after a small number of samples and its behavior is largely unaffected as we increase the value of $\epsilon$, while the DFO algorithm becomes considerably slower when $\epsilon$ is large.

\textbf{Strategic classification simulator.}
We next consider experiments on the credit scoring simulator from~\cite{perdomo2020performative}, which has been employed as an empirical
benchmark for performative prediction in several works~\citep{mendler2020stochastic,
drusvyatskiy2020stochastic, brown2020performative}.
The simulator models a strategic classification problem between a bank and
individual agents seeking a loan. The bank deploys a logistic regression
classifier $f_\theta$ to determine the individuals' default probabilities, while
agents strategically manipulate their features to achieve a more favorable
classification. 

More specifically, individuals correspond to feature, label pairs $(x,y)$ drawn
i.i.d. from a base distribution $\D_0$. Given a classifier $f_\theta$, agents
compute a best-response set of features $x_{\mathrm{BR}}$ by solving an
optimization problem.  The bank then observes the manipulated data points
$(x_{\mathrm{BR}},y)\sim\D(\theta)$. For an appropriate choice of the agents' objective function, the distribution map forms a location family, $x_{\mathrm{BR}} = x + \eps\theta$, where $\epsilon$ is a parameter of the agents' objective. It also serves as a measure of performativity, since this distribution map is $\epsilon$-sensitive. As a final remark, we add $\ell_2$-regularization to the logistic loss to ensure strong convexity. See discussion in \cite{perdomo2020performative} and Appendix~\ref{app:experiments} for full details.

\begin{figure}
\begin{center}
    \includegraphics[width=\linewidth]{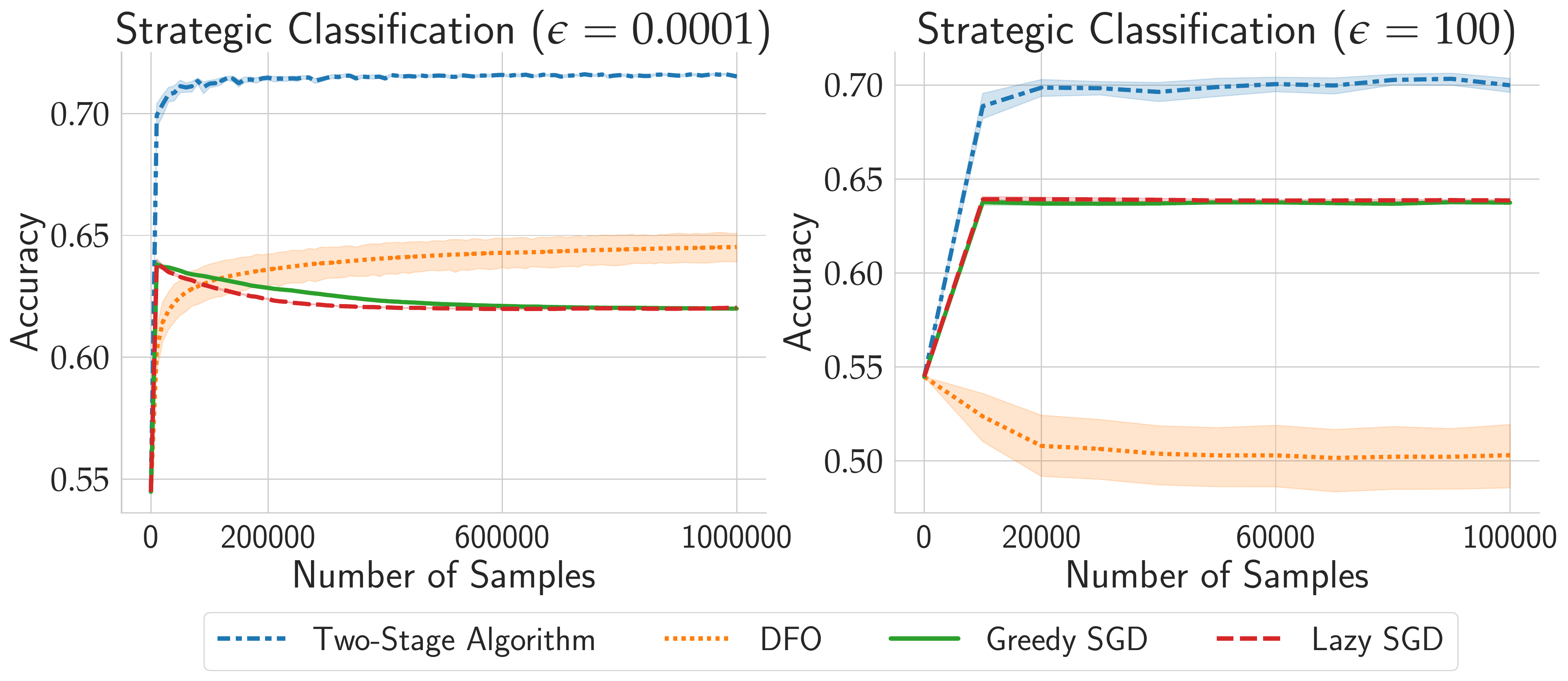}
    \caption{Classification accuracy versus number of samples collected for the two-stage algorithm, DFO algorithm, greedy SGD, and lazy SGD, for $\epsilon = 0.0001\leq \frac{\gamma}{2\beta}$ (left) and $\epsilon = 100 \gg \frac{\gamma}{2\beta}$ (right). 
        Each experiment is repeated 50 times, and we display 95\% bootstrap confidence intervals.
    }
   \figurelabel{strat_class_experiment}
\end{center}	
\end{figure}

Since the logistic loss is not strongly convex in the features, we only have a certificate of convexity when $\epsilon$ is small enough (namely, $\epsilon \leq \frac{\gamma}{2\beta}$). We consider two values of $\epsilon$: one which is below this critical threshold, and one large value for which we do not have theoretical guarantees. When $\epsilon$ is small, both the DFO algorithm and the
two-stage method yield significantly higher accuracy solutions compared to the two variants of SGD (see left panel of \figureref{strat_class_experiment}). Together with the linear regression experiments, this observation serves as further evidence that stable points have significantly worse performative risk relative to performative optima, even in regimes where $\epsilon < \gamma / (2\beta)$. Note also that, although both the DFO algorithm and the two-stage algorithm improve upon methods for repeated retraining, the two-stage algorithm converges with significantly fewer samples and
significantly lower variance.  Indeed, a few thousand samples suffice for convergence of the two-stage method, whereas the DFO algorithm has still not fully converged after a million samples. 

Lastly, on the top right plot, we evaluate these methods for $\epsilon  \gg \gamma / (2\beta)$ which is outside the regime of our theoretical analysis. Consequently, we have no convergence guarantees for any of the four algorithms. Despite the lack of guarantees and the increased strength of performative effects, we see that the two-stage procedure achieves only a slightly lower accuracy than in the previous setting. On the other hand, as described in our echo chamber analogy, greedy and lazy SGD rapidly converge to a local minimum and do not significantly improve predictive performance after the 10k sample mark. Despite extensive tuning, we were unable to improve the performance of the DFO algorithm and achieve nontrivial accuracy with this method. 

\section{Discussion and Future Work}

Given the stark difference between performative stability and optimality, the
goal of our work is to identify the first set of conditions and algorithmic
procedures by which one might be able to provably optimize the performative
risk. To this end, we focus on analyzing the problem at a broad level of
generality, identifying simple, structural conditions under which the
optimization problem becomes tractable. 

However, when applying these ideas in practice, there are a number of important considerations determined by the relevant social context that are not explicitly addressed by our theoretical analysis and which we believe are an important direction for future work. In social settings, such as credit scoring or election forecasts, the choice of loss function must  balance predictive accuracy with any externalities that arise from the classifier's impact on the observed distribution. For example, in lending we 
may wish to find a classifier $f_\theta$ that accurately predicts individual defaults, but that also induces a distribution $\D(\theta)$ over which the mean probability of default is low (or which satisfies some other socially desirable property). 

In order to balance between predictive accuracy and other concerns about the
observed distribution, one possibility is to directly incorporate a penalty on
$\D(\theta)$. For instance, if we would like $\D(\theta)$ to satisfy $\E_{z\sim
\D(\theta)}[z] \approx z_\star$, e.g., to control the mean
probability of default, we can modify our loss by including a regularization
term of the form $(z - z_\star)^\top Q (z - z_\star)$, where $Q$ is some PSD
matrix. Importantly, incorporating such a regularization term does not alter the
convexity of the performative risk. More formally, given any loss $\ell$ and
distribution map $\D(\cdot)$, which satisfy the conditions from
\theoremref{convexity-location-fam}, and hence for which $\PR(\theta) =
\ploss{\theta}{\theta}$ is convex, our analysis proves that the regularized
objective,
\begin{align}
\label{eq:regularized_pr}
\PR'(\theta) \defeq \E_{z\sim\D(\theta)} [\ell(z;\theta) + (z - z_\star)^\top Q (z - z_\star)],
\end{align}  
is also convex. Furthermore, this phenomenon does not just hold for quadratics, \theoremref{convexity-location-fam} shows that the performative risk remains convex after we incorporate any convex function $f(z)$ independent of $\theta$.

Reasoning about regularizers of this form illustrates another important difference between performatively stable solutions and performative optima. In particular, the set of stable points is the same for both $\PR(\theta)$ and its regularized version $\PR'(\theta)$ as defined in equation \eqref{eq:regularized_pr}, since the regularization term is \emph{independent} of $\theta$. Therefore, retraining algorithms such as RRM \cite{perdomo2020performative} or greedy/lazy SGD \cite{mendler2020stochastic} essentially ignore any kind of signal provided by these regularizers and converge to the same point on both the regularized and unregularized objective. To find performatively optimal classifiers whose induced distributions satisfy socially desirable criteria, we must directly engage with and anticipate performative effects, rather than just passively retrain.

Lastly, to date, work on performative prediction has mostly studied the problem from a theoretical perspective. We believe that evaluating the ideas and different design choices raised by these papers in the context of specific applications and case studies on performative prediction would be of high value to the community and an exciting direction for future work.

\section*{Acknowledgements}
We thank Moritz Hardt and Celestine Mendler-D\"unner for many helpful
conversations during the course of this project as well as for providing
detailed feedback on a draft of this manuscript. We would also like to thank the anonymous reviewers whose comments helped improve the quality of our work.

This research was generously supported in part by the National Science
Foundation Graduate Research Fellowship Program under Grant No. DGE 1752814.

\bibliographystyle{plain}
\bibliography{refs}

\newpage
\appendix

\section{Background on Stochastic Orders}
\appendixlabel{stochastic_orders}

In this section we provide the necessary preliminaries from the literature on stochastic orders.

First, we recall the notion of the \emph{convex order}: for two random vectors $z, z'\in\R^m$, we say that $z$ is less than $z'$ in the convex order, denoted $z\leq_{cx} z'$, if for all convex functions $g:\R^m\rightarrow\R$, it holds that
$$\E g(z)\leq \E g(z').$$
Using a slight abuse of notation, we will also write $\D_1\leq_{cx} \D_2$ for two distributions $\D_1,\D_2$ when $z\sim\D_1,z'\sim\D_2$ and $z\leq_{cx}z'$.

Therefore, an immediate way to satisfy condition \eqref{ass:mixture} is to assume that the loss function $\ell(z;\theta)$ is convex in $z$, and to require $\D(\alpha \theta + (1-\alpha) \theta')\leq_{cx} \alpha\D(\theta) + (1-\alpha) \D(\theta')$. The latter condition has been long studied in classical statistical literature and many equivalent characterizations are known (see, e.g., \cite{shaked2007stochastic, ross1996stochastic, muller2002comparison}). This leads to the following corollary of \theoremref{convex-pr}.

\begin{corollary}
\label{corollary:convex-dominance}
Suppose that the loss function is $\gamma$-strongly convex in $\theta$ \eqref{ass:s_convexity_theta} and $\beta$-smooth in $z$ \eqref{ass:smoothness_z}, and that the distribution map $\D(\cdot)$ is $\epsilon$-sensitive \eqref{ass:sensitivity}. Further, assume that $\ell(z;\theta)$ is convex in $z$ and that $\D(\alpha \theta + (1-\alpha) \theta')\leq_{cx} \alpha\D(\theta) + (1-\alpha) \D(\theta')$. Then, the performative risk $\PR(\theta)$ is $(\gamma-2\epsilon\beta)$-convex.
\end{corollary}

Now we discuss important families of distributions that satisfy the convex order condition $\D(\alpha \theta + (1-\alpha) \theta')\leq_{cx} \alpha\D(\theta) + (1-\alpha) \D(\theta')$.

\begin{example}
\label{example:bernoulli}
An obvious example where $\D\left(\alpha\theta + (1-\alpha)\theta'\right) \leq_{cx} \alpha \D(\theta) + (1-\alpha) \D(\theta')$ is when $\D(\alpha\theta + (1-\alpha)\theta') = \alpha \D(\theta) + (1-\alpha)\D(\theta')$. An important setting which satisfies this linearity property is when the probability of a positive outcome of a binary variable is linear in $\theta$: $z_\theta \sim \text{Bern}\left(a+w^\top\theta\right)$ defines $z_\theta\sim\D(\theta)$. In this case, $\D(\alpha\theta+(1-\alpha)\theta') = \alpha\D(\theta) + (1-\alpha) \D(\theta')$.
\end{example}

For further examples, we invoke a convenient characterization of the convex order condition.

\begin{lemma}[\cite{muller2001optimal}]
\label{lem:convex_dominance}
Two random vectors $z$ and $z'$ satisfy $z\leq_{cx} z'$ if and only if there exists a coupling of $z$ and $z'$ such that $\E[z'|z]=z$ a.s.
\end{lemma}

By applying \lemmaref{convex_dominance}, we show that the important case of \emph{location-scale families} satisfies the convex order condition. Therefore, if the loss function is additionally convex in $z$, condition~\eqref{ass:mixture} follows.

\begin{proposition}
\proplabel{sl-weak-convexity}
Suppose that $\D(\theta)$ forms a location-scale family \eqref{eqn:scale-loc-def} such that $\Sigma_0 + \Sigma(\theta)$ has full rank for all $\theta\in\Theta$. Then, $\D(\alpha\theta+(1-\alpha)\theta')\leq_{cx} \alpha\D(\theta) + (1-\alpha)\D(\theta')$ for all $\theta,\theta'\in\Theta$.
\end{proposition}

\begin{proof}
We will construct a coupling $(z,z')$ such that $z\sim \D(\alpha\theta+(1-\alpha)\theta'), z'\sim \alpha\D(\theta) + (1-\alpha)\D(\theta')$, and $\E[z'|z]=z$. Let $z\sim \D(\alpha\theta+(1-\alpha)\theta')$; then we define $z'$ in terms of $z$ as
\begin{equation}
\label{eqn:location-scale-coupling}
z' = (\Sigma_0 + \Sigma(G))(\Sigma_0 + \Sigma(\alpha\theta + (1-\alpha)\theta'))^{-1}\left(z - \mu_0 - \mu(\alpha\theta + (1-\alpha)\theta')\right) + \mu_0 + \mu G,
\end{equation}
where
$$G = \begin{cases}
\theta, \text{ with probability } \alpha,\\
\theta', \text{ with probability } 1-\alpha
\end{cases}$$
is independent of $z$. Notice that
\begin{align*}
    \E[z'~|~z] &= \E\left[(\Sigma_0 + \Sigma(G))(\Sigma_0 + \Sigma(\alpha\theta + (1-\alpha)\theta'))^{-1}\left(z - \mu_0 - \mu(\alpha\theta + (1-\alpha)\theta')\right) + \mu_0 + \mu G~|~ z\right]\\
    &= (\Sigma_0 + \E[\Sigma(G)])(\Sigma_0 + \Sigma(\alpha\theta + (1-\alpha)\theta'))^{-1}\left(z - \mu_0 - \mu(\alpha\theta + (1-\alpha)\theta')\right) + \mu_0 + \E[\mu G]\\
    &= z,
\end{align*}
which follows by linearity of $\mu$ and $\Sigma(\cdot)$ and the fact that $\E[G] = \alpha\theta + (1-\alpha)\theta'$.

We now only need to verify that $z' \sim \alpha\D(\theta) + (1-\alpha)\D(\theta')$ in order to apply \lemmaref{convex_dominance} and conclude that $z'\leq_{cx} z$. Indeed, with probability $\alpha$ we have $G = \theta$, and on that event $z' \stackrel{d}{=} (\Sigma_0 + \Sigma(\theta))\zb + \mu_0 + \mu\theta$; a similar argument applies to $\theta'$. Therefore, putting everything together we conclude that $z\leq_{cx} z'$.	
\end{proof}

\propref{sl-weak-convexity} implies that for all convex functions $g:\R^m\rightarrow \R$,
$$\E_{z\sim \D(\alpha\theta+(1-\alpha)\theta')}[g(z)] \leq \E_{z\sim \alpha\D(\theta) + (1-\alpha)\D(\theta')}[g(z)].$$
We now show that for \emph{strongly} convex $g$, this conclusion can be made even stronger. This result allows for deriving a tighter version of \theoremref{convex-pr} for the important class of location-scale families, stated in \theoremref{convexity-location-fam}.

\begin{proposition}
\label{prop:sl-fam-strong-convexity}
Let $g:\R^m\rightarrow\R$ be a $\gamma_z$-strongly convex function for some $\gamma_z\geq 0$, and let $\D(\theta)$ form a location-scale family \eqref{eqn:scale-loc-def}. Then,
$$\E_{z\sim \D(\alpha\theta+(1-\alpha)\theta')}[g(z)] \leq \E_{z\sim \alpha\D(\theta) + (1-\alpha)\D(\theta')}[g(z)] - \frac{\alpha(1-\alpha)\gamma_z}{2}\E\|\Sigma(\theta - \theta')\zb+ \mu(\theta-\theta')\|_2^2.$$
\end{proposition}

\begin{proof}
Since $g$ is strongly convex, we can write $g(z) = g_0(z) + \frac{\gamma_z}{2}\|z\|_2^2$, where $g_0$ is a convex function. Thus, we want to prove
\begin{align*}
\E_{z\sim \D(\alpha\theta+(1-\alpha)\theta')}\left[g_0(z) + \frac{\gamma_z}{2}\|z\|_2^2\right] &\leq \E_{z'\sim \alpha\D(\theta) + (1-\alpha)\D(\theta')}\left[g_0(z') + \frac{\gamma_z}{2}\|z'\|_2^2\right]\\
&- \frac{\alpha(1-\alpha)\gamma_z}{2}\E\left\|\Sigma(\theta - \theta')\zb+ \mu(\theta-\theta')\right\|_2^2.
\end{align*}
By \propref{sl-weak-convexity}, we know that
$$\E_{z\sim \D(\alpha\theta + (1-\alpha)\theta')}[g_0(z)] \leq \E_{z\sim \alpha\D(\theta) + (1-\alpha)\D(\theta')}[g_0(z)].$$
Therefore, we only need to argue that
$$\E\left[\|z'\|_2^2 - \|z\|_2^2\right] \geq  \alpha(1-\alpha)\E\|\Sigma(\theta - \theta')\zb+ \mu(\theta-\theta')\|_2^2.$$
Without loss of generality, we take $z,z'$ to be coupled as in equation \eqref{eqn:location-scale-coupling}. Then, we can write
\begin{align*}
    \E\left[\|z'\|_2^2 - \|z\|_2^2\right] &= \E\left[\|z' - z\|_2^2 + 2(z' - z)^\top z\right]\\
    &= \E\left[\|z' - z\|_2^2\right]\\
    &= \E\left[\left\| \Sigma\left(G - (\alpha\theta + (1-\alpha)\theta')\right)\zb + \mu\left(G - (\alpha\theta + (1-\alpha)\theta')\right)\right\|_2^2\right],
    \end{align*}
where the second steps follows by iterating expectations, because $\E[z'|z] = z$.

By further taking an expectation over $G$, we get:
\begin{align*}
    &\E\left[\left\| \Sigma\left(G - (\alpha\theta + (1-\alpha)\theta')\right)\zb + \mu\left(G - (\alpha\theta + (1-\alpha)\theta')\right)\right\|_2^2\right]\\
    &= \alpha (1-\alpha)^2\E\|\Sigma(\theta - \theta')\zb + \mu(\theta-\theta')\|_2^2 + (1-\alpha) \alpha^2\E\|\Sigma(\theta' - \theta)\zb+ \mu(\theta-\theta')\|_2^2\\
    &= \alpha(1-\alpha)\E\|\Sigma(\theta - \theta')\zb+ \mu(\theta-\theta')\|_2^2.
\end{align*}	
\end{proof}

\section{Distinguishing between Static and Performative Variables}
\appendixlabel{static_vs_perf}

 In many natural examples, the performative effects are only present in a subset of the variables that make up $z$. For example, in strategic classification, the performative effects are often only present in the strategically manipulated features, and not in the label. In \exampleref{linear_reg}, on the other hand, the effects are only present in the label. For simplicity of exposition, we suppress this distinction between \emph{performative} and \emph{static} variables, that is, those whose distribution does not change for different $\D(\theta)$. However, the reader should think of all assumptions on $z$, such as strong convexity or various Lipschitz assumptions, as only having to apply to the performative variables, while the static ones can be averaged out. To give one example, suppose that $z=(z_s,z_p)$, where $z_s$ denotes the static variables and $z_p$ denotes the performative ones. Using this distinction, the step in equation \eqref{eqn:duality_step} would proceed as follows:
 \begin{align*}
\lefteqn{\E_{(z_s,z_p)\sim\D(\theta)} [\nabla_\theta \ell((z_s,z_p);\theta)]^\top(\theta'-\theta) - \E_{(z_s,z_p')\sim\D(\theta')} [\nabla_\theta \ell((z_s,z_p');\theta)]^\top(\theta'-\theta)}\\
&= \E_{z_s} \left[\left(\E[\nabla_\theta \ell((z_s,z_p);\theta) |z_s]] -  \E[\nabla_\theta \ell((z_s,z_p');\theta)|z_s]\right)^\top(\theta'-\theta)\right]\\
&\leq \E_{z_s}[\beta(z_s) \epsilon(z_s)]~ \|\theta-\theta'\|_2^2.
\end{align*}
Here, $\beta(z_s)$ is the Lipschitz constant of $\nabla_\theta \ell((z_s,\cdot);\theta)$, and $\epsilon(z_s)$ is the sensitivity parameter of the distribution of $z_p$, conditional on $z_s$. As clear from the above example, stating all conditions and proofs while emphasizing this distinction is fairly cumbersome, so we opted for a simplified presentation. Similar calculations can be carried out for the rest of the proofs of the structural results.

\section{Deferred Proofs}

\subsection{Convexity of the Performative Risk}

\paragraph{Proof of \theoremref{convex-pr}.}
We begin by writing out the gradient of the performative risk:
\begin{align*}
    \nabla_\theta \PR(\theta) = \nabla_\theta \left(\int \ell(z;\theta)p_\theta(z)dz\right) &= \int \nabla_\theta \ell(z;\theta)p_\theta(z)dz + \int  \ell(z;\theta)\nabla_\theta p_\theta(z)dz\\
    &= \int \nabla_\theta \ell(z;\theta)p_\theta(z)dz + \int  \ell(z;\theta)\nabla_\theta \log(p_\theta(z))p_\theta(z) dz\\
    &= \E_{z\sim\D(\theta)} [\nabla_\theta \ell(z;\theta)] + \E_{z\sim\D(\theta)} [\ell(z;\theta) \nabla_\theta \log(p_\theta(z))].
\end{align*}
By the first-order condition for convexity, we know that $\PR(\theta)$ is $(\gamma - 2\epsilon\beta)$-convex if and only~if
\begin{align}
\label{eqn:convexity-iff}
\left(\E_{z\sim\D(\theta)} [\nabla_\theta \ell(z;\theta) + \ell(z;\theta) \nabla_\theta \log(p_\theta(z))]\right)^\top(\theta'-\theta) + \frac{\gamma - 2\epsilon\beta}{2}\|\theta-\theta'\|_2^2 \leq \PR(\theta') - \PR(\theta),
\end{align}
for all $\theta,\theta'\in\Theta$. By assumption \eqref{ass:mixture}, we know that for all $\theta,\theta',\theta_0\in\Theta$,
$$\E_{z\sim \D\left(\alpha\theta + (1-\alpha) \theta'\right)}[\ell(z;\theta_0)] \leq \alpha\E_{z\sim \D(\theta)}[\ell(z;\theta_0)] + (1-\alpha)\E_{z\sim \D(\theta')}[\ell(z;\theta_0)].$$
This assumption is equivalent to saying that $g_{\theta_0}(\theta) = \E_{z\sim \D\left(\theta\right)}[\ell(z;\theta_0)]$ is a convex function of $\theta$, for all $\theta_0$. We can express this convexity condition using the equivalent first-order characterization:
$$\E_{z\sim\D(\theta)} [\ell(z;\theta_0) \nabla_\theta \log(p_\theta(z))]^\top(\theta'-\theta) \leq \E_{z\sim \D(\theta')}[\ell(z;\theta_0)] - \E_{z\sim \D(\theta)}[\ell(z;\theta_0)].$$
Since the mixture dominance condition holds for all $\theta, \theta'$ and $\theta_0$, we can set $\theta_0$ equal to $\theta$ in the inequality above to conclude that 
$$\E_{z\sim\D(\theta)} [\ell(z;\theta) \nabla_\theta \log(p_\theta(z))]^\top(\theta'-\theta) \leq \E_{z\sim \D(\theta')}[\ell(z;\theta)] - \E_{z\sim \D(\theta)}[\ell(z;\theta)].$$
Going back to equation \eqref{eqn:convexity-iff}, we see that a sufficient condition for $(\gamma - 2\epsilon\beta)$-convexity of the performative risk is
$$\E_{z\sim\D(\theta)} [\nabla_\theta \ell(z;\theta)]^\top(\theta'-\theta) + \frac{\gamma - 2\epsilon\beta}{2}\|\theta-\theta'\|_2^2 \leq \E_{z\sim\D(\theta')}\ell(z;\theta') - \E_{z\sim\D(\theta')}\ell(z;\theta).$$
By the assumption that the loss is $\gamma$-strongly convex in $\theta$, we know
$$\E_{z\sim\D(\theta')}\ell(z;\theta') - \E_{z\sim\D(\theta')}\ell(z;\theta) \geq \E_{z\sim\D(\theta')}[\nabla_\theta \ell(z;\theta)]^\top(\theta'-\theta) + \frac{\gamma}{2}\|\theta-\theta'\|_2^2,$$
and thus we have further simplified the sufficient condition to
$$\E_{z\sim\D(\theta)} [\nabla_\theta \ell(z;\theta)]^\top(\theta'-\theta) - \E_{z\sim\D(\theta')} [\nabla_\theta \ell(z;\theta)]^\top(\theta'-\theta) \leq \frac{2\epsilon\beta}{2}\|\theta-\theta'\|_2^2.$$
Since the loss is $\beta$-smooth in $z$, we have that $\nabla_\theta \ell(z;\theta)^\top(\theta'-\theta)$ is $\beta \|\theta - \theta'\|_2$-Lipschitz in $z$. Now, we can use the fact that the distribution map is $\epsilon$-sensitive to upper bound the left-hand side by applying the Kantorovich-Rubinstein duality theorem:
\begin{align}
\label{eqn:duality_step}
\E_{z\sim\D(\theta)} [\nabla_\theta \ell(z;\theta)]^\top(\theta'-\theta) - \E_{z\sim\D(\theta')} [\nabla_\theta \ell(z;\theta)]^\top(\theta'-\theta) \leq \epsilon \beta \|\theta-\theta'\|_2^2.
\end{align}
Therefore, we can conclude that the performative risk is $(\gamma - 2\epsilon\beta)$-convex.

\paragraph{Proof of \theoremref{convexity-location-fam}.}
Following the steps of \theoremref{convex-pr}, we know that $\PR(\theta)$ is $\lambda$-convex if and only if
\begin{align*}
\E_{z\sim\D(\theta)} [\nabla_\theta \ell(z;\theta)]^\top(\theta'-\theta) + \E_{z\sim\D(\theta)} [\ell(z;\theta) \nabla_\theta \log(p_\theta(z))]^\top(\theta'-\theta) +\frac{\lambda}{2}\|\theta-\theta'\|_2^2 \leq \PR(\theta') - \PR(\theta),
\end{align*}
for all $\theta,\theta'\in\Theta$.

We now state a technical lemma which rephrases the conclusion of \propref{sl-fam-strong-convexity} in an equivalent way, deferring its proof to the end of this section.
\begin{lemma}
\lemmalabel{sc-identity}
Suppose that
$$\E_{z\sim \D(\alpha\theta+(1-\alpha)\theta')}[g(z)] \leq \E_{z\sim \alpha\D(\theta) + (1-\alpha)\D(\theta')}[g(z)] - \frac{\alpha(1-\alpha)\gamma_z}{2}\E\|\Sigma(\theta - \theta')\zb+ \mu(\theta-\theta')\|_2^2.$$
Then,
$$\E_{z\sim \D(\theta')}[g(z)] \geq \E_{z\sim \D(\theta)}[g(z)] + (\nabla_\theta \E_{z\sim \D(\theta)}[g(z)])^\top (\theta'-\theta) + \frac{\gamma_z}{2}\E\|\Sigma(\theta - \theta')\zb+ \mu(\theta-\theta')\|_2^2.$$
\end{lemma}

Therefore, by \propref{sl-fam-strong-convexity} and \lemmaref{sc-identity}, we know
\begin{align*}
\E_{z\sim\D(\theta)} [\ell(z;\theta) \nabla_\theta \log(p_\theta(z))]^\top(\theta'-\theta)  &\leq \E_{z\sim \D(\theta')}[\ell(z;\theta)] - \E_{z\sim \D(\theta)}[\ell(z;\theta)]\\
 &- \frac{\gamma_z}{2}\E\|\Sigma(\theta - \theta')\zb+ \mu(\theta-\theta')\|_2^2,
\end{align*}
where we take $g(z) = \ell(z;\theta)$.

Thus it suffices to show
$$\E_{z\sim\D(\theta)} [\nabla_\theta \ell(z;\theta)]^\top(\theta'-\theta) +\frac{\lambda}{2}\|\theta-\theta'\|_2^2 \leq \E_{z\sim\D(\theta')}\ell(z;\theta') - \E_{z\sim\D(\theta')}\ell(z;\theta) + \frac{\gamma_z}{2}\E\|\Sigma(\theta - \theta')\zb+ \mu(\theta-\theta')\|_2^2.$$
By the assumption that the loss is $\gamma$-strongly convex, we know
$$\E_{z\sim\D(\theta')}\ell(z;\theta') - \E_{z\sim\D(\theta')}\ell(z;\theta) \geq \E_{z\sim\D(\theta')}[\nabla_\theta \ell(z;\theta)]^\top(\theta'-\theta) + \frac{\gamma}{2}\|\theta-\theta'\|_2^2.$$
With this, we have simplified the sufficient condition for $\gamma$-convexity to
\begin{align}
\label{eqn:smoothness_step}
(\E_{z\sim\D(\theta)} [\nabla_\theta \ell(z;\theta)] - \E_{z\sim\D(\theta')} [\nabla_\theta \ell(z;\theta)])^\top(\theta'-\theta) &\leq \frac{\gamma-\lambda}{2}\|\theta-\theta'\|_2^2 + \frac{\gamma_z}{2}\E\|\Sigma(\theta - \theta')\zb+ \mu(\theta-\theta')\|_2^2.
\end{align}
We bound the left-hand side by applying smoothness of the loss together with the Kantorovich-Rubinstein duality theorem; for this, we need a bound on $W(\D(\theta),\D(\theta'))$. We will use the bound implied by $\epsilon$-sensitivity, as well as the bound implied by the following lemma.
\begin{lemma}
\lemmalabel{wass-dist-sc-fam}
Suppose that the distribution map $\D(\theta)$ forms a location-scale family \eqref{eqn:scale-loc-def}. Then,
$$W(\D(\theta),\D(\theta')) \leq \E\|\Sigma(\theta-\theta')\zb + \mu(\theta-\theta')\|_2.$$
\end{lemma}

\begin{proof}[Proof of \lemmaref{wass-dist-sc-fam}]
By definition, $W(\D(\theta),\D(\theta')) = \inf_{\Pi(\D(\theta),\D(\theta'))} \E_{(z_\theta,z_{\theta'})\sim \Pi(\D(\theta),\D(\theta'))}[\|z_\theta - z_{\theta'}\|_2]$, where $\Pi(\D(\theta),\D(\theta'))$ denotes a coupling of $\D(\theta)$ and $\D(\theta')$. The simplest way to couple $\D(\theta)$ and $\D(\theta')$, or equivalently $z_\theta$ and $z_{\theta'}$, is to sample $\zb\sim \D$, and set $z_\theta = (\Sigma_0 + \Sigma(\theta))\zb + \mu_0 + \mu(\theta)$ and $z_{\theta'} = (\Sigma_0 + \Sigma(\theta'))\zb + \mu_0 + \mu(\theta')$. With this choice, $\|z_\theta - z_{\theta'}\|_2 = \|\Sigma(\theta-\theta')\zb + \mu(\theta-\theta')\|_2$, and hence $W(\D(\theta),\D(\theta'))\leq \E\|\Sigma(\theta-\theta')\zb + \mu(\theta-\theta')\|_2$.
\end{proof}
Therefore, the left-hand side in equation \eqref{eqn:smoothness_step} can be bounded by
$$\E_{z\sim\D(\theta)} [\nabla_\theta \ell(z;\theta)]^\top(\theta'-\theta) - \E_{z\sim\D(\theta')} [\nabla_\theta \ell(z;\theta)]^\top(\theta'-\theta) \leq \beta \E\|\Sigma(\theta-\theta')\zb + \mu(\theta-\theta')\|_2 \|\theta'-\theta\|_2,$$
but also by applying $\epsilon$-sensitivity
$$\E_{z\sim\D(\theta)} [\nabla_\theta \ell(z;\theta)]^\top(\theta'-\theta) - \E_{z\sim\D(\theta')} [\nabla_\theta \ell(z;\theta)]^\top(\theta'-\theta) \leq \beta \epsilon \|\theta'-\theta\|_2^2.$$
Finally, to show $\lambda=\max\left\{\gamma - \beta^2/\gamma_z,~ \gamma + \gamma_z(\sigma_{\min}^2(\mu) + \sigma_{\min}^2(\Sigma)) - 2\beta\epsilon\right\}$-convexity it suffices to show both
\begin{align}
\label{eqn:convexity-cond1}
\beta \E\|\Sigma(\theta-\theta')\zb + \mu(\theta-\theta')\|_2 \|\theta'-\theta\|_2 &\leq \frac{\beta^2/\gamma_z}{2}\|\theta-\theta'\|_2^2 + \frac{\gamma_z}{2}\E\|\Sigma(\theta-\theta')\zb + \mu(\theta-\theta')\|_2^2
\end{align}
and 
\begin{align}
\label{eqn:convexity-cond2}
\beta \epsilon \|\theta'-\theta\|_2^2 &\leq \frac{2\beta\epsilon - \gamma_z(\sigma_{\min}^2(\mu) + \sigma_{\min}^2(\Sigma))}{2}\|\theta-\theta'\|_2^2 + \frac{\gamma_z}{2}\E\|\Sigma(\theta-\theta')\zb + \mu(\theta-\theta')\|_2^2.
\end{align}
By the AM-GM inequality, we have
$$\beta \E\|\Sigma(\theta-\theta')\zb + \mu(\theta-\theta')\|_2 \|\theta'-\theta\|_2 \leq \frac{1}{2}\frac{\beta^2}{\gamma_z}\|\theta'-\theta\|_2^2 + \frac{\gamma_z}{2}\E\|\Sigma(\theta-\theta')\zb + \mu(\theta-\theta')\|_2^2,$$
and so condition \eqref{eqn:convexity-cond1} follows.

For condition \eqref{eqn:convexity-cond2}, we observe that 
\begin{align*}
	\E\|\Sigma(\theta-\theta')\zb + \mu(\theta-\theta')\|_2^2 &= \E\|\Sigma(\theta-\theta')\zb\|_2^2 + \|\mu(\theta-\theta')\|_2^2\\
	&= \tr\left(\Sigma(\theta-\theta')\Sigma_{\zb}\Sigma(\theta-\theta')^\top\right) + \|\mu(\theta-\theta')\|_2^2\\
	&= \|\Sigma_{\zb}^{1/2}\Sigma(\theta-\theta')^\top\|_F^2 + \|\mu(\theta-\theta')\|_2^2.
\end{align*}
Applying $\sigma_{\min}(\Sigma)\|\theta-\theta'\|_2\leq \|\Sigma_{\zb}^{1/2}\Sigma(\theta-\theta')^\top\|_F$ and $\sigma_{\min}(\mu)\|\theta-\theta'\|_2\leq \|\mu(\theta-\theta')\|_2$ completes the proof of the theorem.

\begin{proof}[Proof of \lemmaref{sc-identity}]
The proof follows the standard argument for proving equivalent formulations of strong convexity.

First we show that $\E_{z\sim \D(\theta)}[g(z)] - \frac{\gamma_z}{2}\E\|\Sigma(\theta)\zb+ \mu\theta\|_2^2$ is convex in $\theta$. This follows because:
\begin{align*}
&\E_{z\sim \D(\alpha\theta + (1-\alpha)\theta')}[g(z)] - \frac{\gamma_z}{2}\E\|\Sigma(\alpha\theta + (1-\alpha)\theta')\zb+ \mu(\alpha\theta+(1-\alpha)\theta')\|_2^2\\
&\quad \quad \leq \E_{z\sim \alpha\D(\theta) + (1-\alpha)\D(\theta')}[g(z)] - \frac{\alpha(1-\alpha)\gamma_z}{2}\E\|\Sigma(\theta - \theta')\zb+ \mu(\theta-\theta')\|_2^2\\
 &\quad \quad- \frac{\gamma_z}{2}\E\|\Sigma(\alpha\theta + (1-\alpha)\theta')\zb+ \mu(\alpha\theta+(1-\alpha)\theta')\|_2^2\\
&\quad \quad = \E_{z\sim \alpha\D(\theta) + (1-\alpha)\D(\theta')}[g(z)] - \frac{\gamma_z}{2}\alpha^2\E\|\Sigma(\theta)\zb+ \mu\theta\|_2^2 - \frac{\gamma_z}{2}(1-\alpha)^2\E\|\Sigma(\theta')\zb+ \mu\theta'\|_2^2\\
 &\quad \quad+\frac{\gamma_z}{2} 2\alpha(1-\alpha) \E(\Sigma(\theta) + \mu\theta)^\top (\Sigma(\theta') + \mu\theta') - \frac{\alpha(1-\alpha)\gamma_z}{2}\E\|\Sigma(\theta - \theta')\zb+ \mu(\theta-\theta')\|_2^2\\
 &\quad \quad = \E_{z\sim \alpha\D(\theta) + (1-\alpha)\D(\theta')}[g(z)] - \frac{\gamma_z}{2}\alpha\E\|\Sigma(\theta)\zb+ \mu\theta\|_2^2 - \frac{\gamma_z}{2}(1-\alpha)\E\|\Sigma(\theta')\zb+ \mu\theta'\|_2^2\\
 &\quad \quad = \alpha\left(\E_{z\sim\D(\theta)}[g(z)] - \frac{\gamma_z}{2}\E\|\Sigma(\theta)\zb+ \mu\theta\|_2^2\right) - (1-\alpha)\left(\E_{z\sim\D(\theta')}[g(z)]\frac{\gamma_z}{2}\E\|\Sigma(\theta')\zb+ \mu\theta'\|_2^2\right).
\end{align*}
By the equivalent first-order characterization, this means that
\begin{align*}
\E_{z\sim \D(\theta')}[g(z)] &\geq \frac{\gamma_z}{2}\E\|\Sigma(\theta')\zb+ \mu\theta'\|_2^2 + \E_{z\sim \D(\theta)}[g(z)] - \frac{\gamma_z}{2}\E\|\Sigma(\theta)\zb+ \mu\theta\|_2^2\\
 &+ (\nabla_\theta \E_{z\sim\D(\theta)} [g(z)])^\top(\theta'-\theta) - \frac{\gamma_z}{2}2\E(\Sigma(\theta)\zb +\mu\theta)^\top(\nabla_\theta (\Sigma(\theta)\zb +\mu\theta))^\top(\theta'-\theta)\\
 &\geq \frac{\gamma_z}{2}\E\|\Sigma(\theta')\zb+ \mu\theta'\|_2^2 + \E_{z\sim \D(\theta)}[g(z)] - \frac{\gamma_z}{2}\E\|\Sigma(\theta)\zb+ \mu\theta\|_2^2\\
 &+ (\nabla_\theta \E_{z\sim\D(\theta)} [g(z)])^\top(\theta'-\theta) - \gamma_z \E(\Sigma(\theta)\zb +\mu\theta)^\top(\Sigma(\theta'-\theta)\zb +\mu(\theta'-\theta))\\
 &= \E_{z\sim \D(\theta)}[g(z)] + (\nabla_\theta \E_{z\sim \D(\theta)}[g(z)])^\top (\theta'-\theta) + \frac{\gamma_z}{2}\E\|\Sigma(\theta - \theta')\zb+ \mu(\theta-\theta')\|_2^2.
\end{align*}
\end{proof}

\begin{remark}
\remarklabel{eps_bound}
We note that the sensitivity parameter $\epsilon$ can be bounded in terms of the location and scale parameters for location-scale families. In particular, in showing condition \eqref{eqn:convexity-cond2}, we saw that
$$\E\|\Sigma(\theta-\theta')\zb + \mu(\theta-\theta')\|_2^2 = \|\Sigma_{\zb}^{1/2}\Sigma(\theta-\theta')^\top\|_F^2 + \|\mu(\theta-\theta')\|_2^2.$$	
If we then denote
\begin{align*}
\sigma_{\max}(\mu)=\max_{\|\theta\|_2=1}\|\mu\theta\|_2, \quad \sigma_{\max}(\Sigma) &=\max_{\|\theta\|_2=1}\|\Sigma_{\zb}^{1/2}\Sigma(\theta)^\top\|_F,
\end{align*}
we can see that
$\E\|\Sigma(\theta-\theta')\zb + \mu(\theta-\theta')\|_2^2 \leq \sigma_{\max}^2(\mu)\|\theta-\theta'\|_2^2 + \sigma_{\max}^2(\Sigma)\|\theta-\theta'\|_2^2$. Combining this result with \lemmaref{wass-dist-sc-fam} and Jensen's inequality, we get that
$$W(\D(\theta),\D(\theta')) \leq \sqrt{\sigma_{\max}^2(\mu) + \sigma_{\max}^2(\Sigma)}\|\theta-\theta'\|_2,$$
and so $\epsilon\leq \sqrt{\sigma_{\max}^2(\mu) + \sigma_{\max}^2(\Sigma)}$.
 
\end{remark}

\subsection{Two-Stage Algorithm for Location Families}
\appendixlabel{model_based_analysis}

We carefully review the problem setup and introduce the remaining assumptions. The distribution map $\D$ parameterizes a location family
\begin{align*}
    z_\theta \sim \D(\theta)~ \Leftrightarrow~ z_\theta\stackrel{d}{=} \zb + \mu\theta,
\end{align*}
where $\zb\sim\D_0$. We assume the base distribution $\D_0$ is zero-mean and subgaussian with parameter $K$. The loss function $\ell(z; \theta)$ is $L_z$-Lipschitz in $z$, $L$-Lipschitz and in $\theta$, and $\beta$-smooth in $(z, \theta)$ in the sense that $\nabla \ell(z;
\theta) \in \R^{m+d}$ is Lipschitz in $(z, \theta)$.

We also assume that $\lambda = \max\{\gamma - \beta^2/\gamma_z, \gamma - 2\epsilon\beta + \gamma_z\sigma^2_{\min}(\mu)\}>0$, where $\gamma$ and $\gamma_z$ are the strong convexity parameters of the loss in $\theta$ and $z$, respectively. By \theoremref{convexity-location-fam}, this implies
that the performative risk is $\lambda$-strongly convex.


We assume that the performative optimum $\thetaPO$ is contained in a ball of radius $R$, so in the second stage we can set the domain of optimization to be $\Theta=\{\theta:\|\theta\|_2\leq R\}$.
Finally, we assume that the minimizer of the perturbed
performative risk at the population level, $\thetahat \in \argmin_{\theta \in
\Theta} \PRh(\theta)$ is contained in the interior of $\Theta$ with
probability 1.

\begin{theorem}
    \theoremlabel{location_model_based}
    Under the preceding assumptions, if $n \geq \Omega\paren{d + m +
    \log(1/\delta)}$, then, with probability $1-\delta$,
    Algorithm~\ref{alg:model_based_location_fam} returns a point $\thetahat_n$
    such that
    \begin{align*}
        \PR(\thetahat_n) - \PR(\thetaPO)
        \leq O\paren{\frac{d + m + \log(1/\delta)}{n} 
                     + \frac{1}{\delta n}}.
    \end{align*}
\end{theorem}

Before proceeding to the proof of this result, we first state four auxiliary
lemmas, which constitute the bulk of our analysis. The proofs of the lemmas are
included in~\appendixref{two_stage_lemmas}. The first lemma is a standard result about ordinary least-squares estimation.

\begin{lemma}
    \lemmalabel{least_squares}
    If $n \geq \Omega(d + m + \log(1/\delta))$, then with probability $1-\delta$,
    \begin{align*}
        \norm{\mu - \estmu}
        \leq O\paren{\sqrt{\frac{(d + m) + \log(1/\delta)}{n}}}.
    \end{align*}
\end{lemma}

The next lemma is a simple adaptation from Theorem 2
in~\cite{shalev2010learnability} controlling the generalization gap of the
empirical risk minimizer for strongly convex losses.

\begin{lemma}
    \lemmalabel{generalization}
    Suppose $\PRh_n$ is $\hat\lambda$-strongly convex. Then, with probability at least $1-\delta$,
    \begin{align*}
        \PRh(\thetahat_n) - \PRh(\thetahat) 
        \leq \frac{4 (L_z \norm{\estmu} + L)^2}{\delta \hat\lambda n}.
    \end{align*}
\end{lemma}

The next lemma controls the difference in gradients between the true
performative risk $\PR$ and the perturbed performative risk $\PRh$.
\begin{lemma}
    \lemmalabel{perturbation_bound}
    For any $\theta \in \Theta$,
    \begin{align*}
        \norm{\grad \PR(\theta) - \grad \PRh(\theta)}_2^2
        \leq O(\norm{\mu}^2 \norm{\mu - \estmu}^2).
    \end{align*}
\end{lemma}

Finally, the last lemma shows that the smoothness assumptions on the loss ensure smoothness of the performative risk. Here, by $\beta_\theta$-smoothness we mean that $\nabla_\theta \PR(\theta)$ is $\beta_\theta$-Lipschitz.
\begin{lemma}
    \lemmalabel{pr_smoothness}
     Under the proceeding assumptions, the performative risk $\PR(\theta)$ is
     $\beta_\theta = O(\norm{\mu}^2)$-smooth.
\end{lemma}

With these lemmas in hand, we are now ready to prove \theoremref{location_model_based}.
\begin{proof}[Proof of~\theoremref{location_model_based}]
By assumption, the performative risk $\PR(\theta)$ is $\lambda$-strongly convex,
for some $\lambda > 0$.  This implies
\begin{align*}
    \PR(\thetahat_n) - \PR(\thetaPO)
    \leq \frac{1}{2\lambda}\norm{\grad \PR(\thetahat_n)}_2^2.
\end{align*}
Since $\thetaPOhat$ is an interior minimizer of $\PRh$, we know $\grad \PRh(\thetaPOhat) = 0$.
Using $\norm{a + b}^2 \leq 2\norm{a}^2 + 2\norm{b}^2$, 
\begin{align}
\label{eqn:PR-grad-triangle_ineq}
    \frac{1}{2\lambda}\norm{\grad \PR(\thetahat_n)}_2^2
    &= \frac{1}{2\lambda}\norm{\grad \PR(\thetahat_n) - \grad \PRh(\thetaPOhat)}_2^2\nonumber \\
    &= \frac{1}{2\lambda} \norm{\grad \PR(\thetahat_n) - \grad \PRh(\thetahat_n)
    +  \grad \PRh(\thetahat_n) - \grad \PRh(\thetaPOhat)}_2^2\nonumber \\
    &\leq \frac{1}{\lambda}\norm{\grad \PR(\thetahat_n) - \grad \PRh(\thetahat_n)}_2^2
    + \frac{1}{\lambda}\norm{\grad \PRh(\thetahat_n) - \grad \PRh(\thetaPOhat)}_2^2.
\end{align}
We bound each of these terms separately. 
For the first term, by~\lemmaref{perturbation_bound}, 
\begin{align*}
    \norm{\grad \PR(\thetahat_n) - \grad \PRh(\thetahat_n)}_2^2
    \leq O(\norm{\mu}^2\norm{\mu - \estmu}^2).
\end{align*}
By~\lemmaref{least_squares}, with probability $1-\delta$, we can bound $\norm{\mu - \estmu}^2 \leq O\paren{\frac{d + m + \log(1/\delta)}{n}}$, and thus
\begin{align*}
    \norm{\grad \PR(\thetahat_n) - \grad \PRh(\thetahat_n)}_2^2
    \leq O\paren{\frac{d + m + \log(1/\delta)}{n}}.
\end{align*}

For the second term in equation \eqref{eqn:PR-grad-triangle_ineq}, 
notice that $\lambda = \max\{\gamma - \beta^2/\gamma_z, \gamma - 2\epsilon\beta + \gamma_z\sigma^2_{\min}(\mu)\} > 0$ 
implies that $\PRh$ is at least
$\hat{\lambda} = \lambda - O(\frac{1}{\sqrt{n}})$-strongly convex. This follows
because $|\sigma_{\min}(\mu) - \sigma_{\min}(\estmu)| \leq \|\mu - \estmu\|$ by
Weyl's inequality (see for example Theorem 3.3.16 in ~\cite{roger1994topics}), and
$\PRh$ is $O(\norm{\estmu})$-sensitive, so by~\lemmaref{least_squares}, each
term depending on $\epsilon$ or $\sigma_{\min}(\estmu)$ is within
$O(1/\sqrt{n})$ or $O(1/n)$ of the corresponding values for the non-perturbed risk $\PR$. 

Hence, when $n \geq \Omega(1/\lambda^2)$, the strong convexity
parameter of the perturbed performative risk, $\hat{\lambda}$, is at least $\lambda / 2$.

With this, we can apply the fact that $\thetaPOhat$ is an interior minimizer of
$\PRh$ by assumption to conclude that when $n \geq \Omega(1/\lambda^2)$,
\begin{align*}
    \norm{\thetahat_n - \thetaPOhat}_2^2 
    \leq \frac{4}{\lambda}\paren{\PRh(\thetahat_n) - \PRh(\thetaPOhat)}.
\end{align*}
Now, when $\PRh$ is strongly convex, the finite-sample performative risk
$\PRh_n$ is also strongly convex because~\theoremref{convexity-location-fam}
does not depend on the base distribution $\D_0$, and $\PRh_n$ is simply $\PRh$
when the base distribution $\D_0$ is replaced with the uniform distribution on
$\set{z_1, \dots, z_n}$. Consequently, by~\lemmaref{generalization}, with
probability $1-\delta$, 
\begin{align*}
    \norm{\thetahat_n - \thetaPOhat}_2^2 
    \leq O\paren{\PRh(\thetahat_n) - \PRh(\thetaPOhat)} \leq O\paren{\frac{\norm{\estmu}^2}{\delta n}}.
\end{align*}
By~\lemmaref{pr_smoothness}, $\PRh$ is $O(\norm{\estmu}^2)$-smooth.
Applying the previous display then gives us,
\begin{align*}
    \norm{\grad \PRh(\thetahat_n) - \grad \PRh(\thetaPOhat)}_2^2
    \leq O\paren{\norm{\estmu}^4 \norm{\thetahat_n - \thetaPOhat}_2^2}
    \leq O\paren{\frac{\norm{\estmu}^6}{\delta n}}.
\end{align*}

By the triangle inequality and repeated application of $(a + b)^2 \leq 2a^2 +
2b^2$, $\norm{\estmu}^6 \leq 128\norm{\estmu - \mu}^6 + 128\norm{\mu}^6$.
Therefore, the above term is $O(\norm{\mu}^6 / \delta n)$.  Putting everything
together with a union bound, we have shown that with probability $1-\delta$, if
$n \geq \Omega(d + m + \log(1/\delta))$, it holds that
\begin{align*}
    \PR(\thetahat_n) - \PR(\thetaPO)
    \leq O\paren{\frac{d + m + \log(1/\delta) }{n} + \frac{1}{\delta n}},
\end{align*}
as desired.
\end{proof}

\subsection{Proofs of Lemmas for Two-Stage Algorithm Analysis}
\appendixlabel{two_stage_lemmas}

The proof of~\lemmaref{least_squares} is essentially standard (see, e.g., \cite{matni2019tutorial}),
but we include it for completeness.

\begin{proof}[Proof of~\lemmaref{least_squares}]
Define $Z \in \R^{n \times m}$ with rows $z_i$ and $\Theta \in \R^{n \times d}$
with rows $\theta_i$, $1\leq i\leq n$. Then, $Z = \Theta \mu^\top + Z_0$, where $Z_0\in\R^{n\times m}$ is a matrix with base samples from $\D_0$ as rows. Temporarily assume that $\Theta^\top \Theta$ is invertible; we
will later condition on this event. Separately optimizing over each row of
$\mu$, we can write the least-squares estimator as
\begin{align*}
    \estmu^\top = \paren{\Theta^\top \Theta}^{-1} \Theta^\top Z.
\end{align*}
Consequently, we can bound the estimation
error as
\begin{align*}
    \norm{\mu - \estmu}
    = \norm{\mu^\top - \estmu^\top}
    &= \norm{\mu^\top - \paren{\Theta^\top \Theta}^{-1}\Theta^\top \paren{\Theta \mu^\top + Z_0}} \\ 
    &= \norm{\paren{\Theta^\top \Theta}^{-1} \Theta^\top Z_0} \\
    &\leq \frac{1}{\lambda_{\min}(\Theta^\top \Theta)} \norm{\Theta^\top Z_0}.
\end{align*}
Since $\theta_i \sim \normal{0}{I}$, $\Theta \in \R^{n \times d}$ has i.i.d.
$\normal{0}{1}$ entries, and so $\Theta^\top \Theta$ is a standard Wishart
matrix. The standard bound on the minimum eigenvalue of a
Wishart matrix (see Theorem 4.6.1 in \cite{vershynin2018high}) gives, with probability $1-\delta$,
\begin{align*}
    \sqrt{\lambda_{\min}(\Theta^\top \Theta)}
    \geq \Omega(\sqrt{n} - \sqrt{d} - \sqrt{\log(1/\delta)}).
\end{align*}
Therefore, if $n \geq \Omega(d + \log(2/\delta))$, then, with probability $1-\delta / 2$,
\begin{align}
\label{eqn:part1}
\sqrt{\lambda_{\min}(\Theta^\top \Theta)} \geq \Omega(\sqrt{n} / 2).
\end{align}

Control of the second term, $\norm{\Theta^\top Z_0}$, also follows from a
standard covering argument followed by the Bernstein bound. Write $\Theta^\top Z_0 =
\sum_{i=1}^n \theta_i (\zb)_i^\top$. Let $\cB^d$ and $\cB^m$ denote the unit
balls in $\R^d$ and $\R^m$, respectively. Then,
\begin{align*}
    \norm{\Theta^\top Z_0}
    &= \sup_{x \in \cB^d, y \in \cB^m} x^\top \paren{\sum_{i=1}^n \theta_i (\zb)_i^\top} y = \sup_{x \in \cB^d, y \in \cB^m} \sum_{i=1}^n \paren{x^\top \theta_i} \paren{(\zb)_i^\top y}.
\end{align*}
Let $\cN_\eps$, and $\cM_\eps$ denote $\eps$-coverings of $\cB^d$ and $\cB^m$,
respectively. A volumetric bound gives $\abs{\cN_\eps} \leq \paren{1 +
\frac{2}{\eps}}^d$ and similarly $\abs{\cM_\eps} \leq \paren{1 +
\frac{2}{\eps}}^m$ (see Corollary 4.2.13 in \cite{vershynin2018high}). Taking $\eps = 1/4$, $\abs{\cN_\eps} \leq 9^d$ and
$\abs{\cM_\eps} \leq 9^m$. Approximating the supremum over the $\eps$-nets gives
\begin{align*}
    \norm{\Theta^\top Z_0}
    &\leq 2\max_{x \in \cN_\eps, y \in \cM_\eps} \sum_{i=1}^n \paren{x^\top \theta_i} \paren{(z_0)_i^\top y}.
\end{align*}
Fix $x, y \in \cN_\eps, \cM_\eps$. Since $\theta_i \sim \normal{0}{I}$ and
$\norm{x}_2 = 1$, $x^\top \theta_i \sim \normal{0}{1}$, which has
subgaussian norm 1.  Similarly, since $(\zb)_i$ is subgaussian with parameter $K$ 
and $\norm{y}_2 = 1$, the marginal $(\zb)_i^\top y$ is subgaussian
with parameter $K$. Since $\zb$ and $\theta$ are independent and zero-mean,
the product $(x^\top \theta_i)((\zb)_i^\top y)$ is zero-mean and subexponential
with parameter $K$. Since each term is subexponential, by the Bernstein
bound (see Theorem 2.8.1 in \cite{vershynin2018high}), for any $t > 0$,
\begin{align*}
    \pr{\sum_{i=1}^n \paren{x^\top \theta_i} \paren{(\zb)_i^\top y} > t /2}
    \leq \exp\paren{-c \min\left\{\frac{t^2}{nK^2}, \frac{t}{K}\right\}},
\end{align*}
for some universal constant $c$. Taking a union bound over the $\eps$-nets, 
\begin{align*}
    \pr{\norm{\Theta^\top Z_0} > t}
    \leq 9^{d + m}\exp\paren{-c \min\left\{\frac{t^2}{nK^2}, \frac{t}{K}\right\}}.
\end{align*}
If $n \geq \Omega\paren{d+m +
\log(2/\delta)}$, then with probability at least $1-\delta/2$,
\begin{align}
\label{eqn:part2}
    \norm{\Theta^\top Z_0} \leq O(\sqrt{n((d + m) + \log(1/\delta))}).
\end{align}
Combining equations \eqref{eqn:part1} and \eqref{eqn:part2} with a union bound, 
if $n \geq \Omega(d + m + \log(1/\delta))$, then
\begin{align*}
    \norm{\mu - \estmu}
    \leq O\left(\sqrt{\frac{(d + m) + \log(1/\delta)}{n}}\right).
\end{align*}
\end{proof}

\begin{proof}[Proof of~\lemmaref{perturbation_bound}]
Under the location-family parameterization, we can write
\begin{align*}
    \PR(\theta)
    = \ploss{\theta}{\theta}
    = \E_{z_0 \sim \D_0}\ell(z_0 + \mu\theta; \theta),
\end{align*}
so the gradients are given by
\begin{align*}
    \grad \PR(\theta)
    = \E_{z_0 \sim \D_0} \grad \ell(z_0 + \mu\theta; \theta)
    \quad
    \text{and}
    \quad
    \grad \PRh(\theta)
    = \E_{z_0 \sim \D_0} \grad \ell(z_0 + \estmu\theta; \theta).
\end{align*}
This representation allows us to write
\begin{align*}
    \left\|\grad \PR(\theta) - \grad \PRh(\theta)\right\|_2^2
    &= \left\|\E_{z_0 \sim \D_0} \brack{\grad \ell(z_0 + \mu\theta;\theta) 
       - \grad \ell(z_0 + \estmu\theta;\theta)}\right\|_2^2.
\end{align*}
Applying the chain rule, together with the triangle-inequality, gives
\begin{align*}
   \left\|\grad \PR(\theta) - \grad \PRh(\theta)\right\|_2
    &\leq \left\|\E_{z_0 \sim \D_0} \brack{\grad_\theta \ell(z_0 + \mu\theta;\theta )
    - \grad_\theta \ell(z_0 + \estmu\theta;\theta)}\right\|_2\\
     &+ \left\|\E_{z_0 \sim \D_0} \brack{\mu^\top\grad_{z} \ell(z_0 + \mu\theta ;\theta)
       - \estmu^\top \grad_{z} \ell(z_0 + \estmu\theta;\theta)}\right\|_2.
\end{align*}
We bound each of these terms separately. For the first term, 
$\beta$-smoothness in $z$ immediately gives
\begin{align*}
    \left\|\E_{z_0 \sim \D_0} \brack{\grad_\theta \ell(z_0 + \mu\theta;\theta)
       - \grad_\theta \ell(z_0 + \estmu\theta;\theta)}\right\|_2
    \leq \beta\norm{\mu\theta - \estmu\theta}_2
    \leq \beta \norm{\mu - \estmu}\norm{\theta}_2.
\end{align*}
For the second term, adding and subtracting
$\mu^\top \grad_{z}\ell(z_0 + \estmu\theta; \theta)$ and then using the triangle
inequality,
\begin{align*}
    &\left\|\E_{z_0 \sim \D_0} \brack{\mu^\top\grad_{z} \ell(z_0 + \mu\theta);\theta)
       - \estmu^\top \grad_{z} \ell(z_0 + \estmu\theta;\theta)}\right\|_2 \\
    &\leq 
    \norm{\mu}\norm{\E_{z_0 \sim \D_0} [\grad_{z} \ell(z_0 + \mu\theta); \theta)
    - \grad_{z} \ell(z_0 + \estmu\theta; \theta)]}_2
    + \norm{\mu - \estmu}\norm{\E_{z_0 \sim \D_0} [\grad_{z} \ell(z_0 + \estmu\theta; \theta)]}_2 \\
    &\leq \beta \norm{\mu}\norm{\mu - \estmu}\norm{\theta}_2
    + L_z\norm{\mu - \estmu},
\end{align*}
where the last line used $\beta$-smoothness in
$z$. Combining both pieces, we have
\begin{align*}
    \left\|\grad \PR(\theta) - \grad \PRh(\theta)\right\|_2
    \leq \paren{\paren{\beta + \beta \norm{\mu}}\norm{\theta}_2 
    + L_z}\norm{\mu - \estmu}.
\end{align*}
Using the trivial bound $\norm{\theta}_2 \leq R$, and then squaring both sides,
\begin{align*}
    \norm{\grad \PR(\thetahat) - \grad \PRh(\thetahat)}_2^2
    \leq \paren{\paren{1 + \norm{\mu}}\beta R + L_z}^2
          \norm{\mu - \estmu}^2.
\end{align*}
\end{proof}

\begin{proof}[Proof of~\lemmaref{pr_smoothness}]
By applying the location family parameterization as in the proof of \lemmaref{perturbation_bound}, we get
\begin{align*}
    \norm{\grad \PR(\theta) - \grad \PR(\theta')}_2
    = \norm{\E_{z_0 \sim \D_0} [\grad \ell(z_0 + \mu\theta; \theta) 
            - \grad \ell(z_0 + \mu\theta'; \theta')]}_2.
\end{align*}
Using the chain rule and the triangle inequality,
\begin{align}
\label{eqn:smoothnes_triangle_ineq}
    \norm{\grad \PR(\theta) - \grad \PR(\theta')}_2
    &\leq 
     \norm{\E_{z_0 \sim \D_0} \grad_\theta \ell(z_0 + \mu\theta; \theta) 
            - \grad_\theta \ell(z_0 + \mu\theta'; \theta')}_2 \nonumber \\
    &+ \norm{\E_{z_0 \sim \D_0} \mu^\top \grad_{z} \ell(z_0 + \mu\theta; \theta) 
            - \mu^\top \grad_{z} \ell(z_0 + \mu\theta'); \theta')}_2.
\end{align}
For the first term in equation \eqref{eqn:smoothnes_triangle_ineq}, adding and subtracting $\grad_\theta \ell(z +
\mu\theta'; \theta)$ and using the triangle inequality gives
\begin{align*}
     \norm{\E_{z_0 \sim \D_0} [\grad_\theta \ell(z_0 + \mu\theta; \theta) 
            - \grad_\theta \ell(z_0 + \mu\theta'; \theta')]}_2
    &\leq 
    \norm{\E_{z_0 \sim \D_0}
           \grad_\theta \ell(z_0 + \mu\theta; \theta) 
            - \grad_\theta \ell(z_0 + \mu\theta'; \theta)}_2 \\
    & + \norm{\E_{z_0 \sim \D_0}
              \grad_\theta \ell(z_0 + \mu\theta'; \theta) 
             - \grad_\theta \ell(z_0 + \mu\theta'; \theta')}_2 \\
    &\leq 
    \beta \norm{\mu}\norm{\theta - \theta'}_2 
    + \beta \norm{\theta - \theta'}_2,
\end{align*}
where we used Jensen's inequality and the assumption that
$\grad_\theta \ell(z; \theta)$ is $\beta$-Lipschitz in $z$
(for the first term) and $\beta$-Lipschitz in $\theta$
(for the second term).

Now, for the second term in equation \eqref{eqn:smoothnes_triangle_ineq}, similarly adding and subtracting $\mu^\top\grad_{z}
\ell(z + \mu\theta'; \theta)$ and using the triangle inequality gives
\begin{align*}
    \norm{\E_{z_0 \sim \D_0} [\mu^\top \grad_{z} \ell(z_0 + \mu\theta; \theta) 
            - \mu^\top \grad_{z} \ell(z_0 + \mu\theta'; \theta')]}_2
    &\leq \norm{\E_{z_0 \sim \D_0}
          \mu^\top \grad_{z} \ell(z_0 + \mu\theta; \theta) 
           - \mu^\top \grad_{z} \ell(z_0 + \mu\theta'; \theta)}_2 \\
    &+ \norm{\E_{z_0 \sim \D_0}
          \mu^\top \grad_{z} \ell(z_0 + \mu\theta'; \theta) 
           - \mu^\top \grad_{z} \ell(z_0 + \mu\theta'; \theta')}_2 \\
    &\leq \beta\norm{\mu}^2 \norm{\theta - \theta'}_2 
    + \beta \norm{\mu}^2 \norm{\theta - \theta'}_2,
\end{align*}
where we used 
$\grad_{z} \ell(z; \theta)$ is $\beta$ Lipschitz in $z$
(for the first term) and $\beta$ Lipschitz in $\theta$
(for the second term). This completes the proof.
\end{proof}

\section{Experimental Details}
\appendixlabel{experiments}

Lastly, we elaborate on the implementation details of the various simulators and algorithms evaluated in \sectionref{experiments}.

\subsection{Synthetic Linear Regression Example}

\paragraph{Data generating process.}
Given a parameter vector $\theta \in \R^d$, as per \exampleref{linear_reg},
feature label pairs $(x,y)$ are generated according to the following data
generating process:
\begin{enumerate}
	\item $x \sim \cN(0, \Sigma_x)$.
	\item $y = \beta^\top x + \mu^\top \theta + U_y$ where $U_y \sim \cN(0, \sigma^2_y)$.
\end{enumerate}
In our experiments, we take $d=20$, and set $\sigma_y^2 = 0.01$. For each trial, we
sample $\Sigma_x$ as a random symmetric positive-definite matrix with operator
norm $0.01$, sample $\beta \sim \cN(0, I_d)$, and sample $w$ uniformly on the
sphere of radius $\eps$, where $\eps$ is the sensitivity parameter of the
distribution map. In our experiments, we choose $\eps \in \set{0.01, 100}$.

For this example, the performative optimum $\thetaPO$ can be computed in
closed-form due to the squared-loss and the linearity of the performative
effects. In particular,
\begin{align*}
    \thetaPO = (\Sigma_x + \mu\mu^\top)^{-1} \Sigma_x\beta. 
\end{align*}

\paragraph{Algorithms.}
We compare four different algorithms. In all four cases, we set $\Theta = \{\theta:\|\theta\|_2\leq 10\}$.
\begin{enumerate}
\item \textbf{The two-stage procedure.} For a budget of $n$ samples, the two-stage procedure consists of
first deploying $n /2$ classifiers $\theta_i \sim \normal{0}{I_d}$ and observing
data $(x_i, y_i)$. We then compute an estimate $\estmu$ by solving a
least-squares problem:
\begin{align*}
    \estmu \in \argmin_{\mu, c} \sum_{i=1}^n (y_i - \mu^\top \theta_i - c)^2.
\end{align*}
After computing $\estmu$, the algorithm collects another $n/2$ samples $(x_i,
y_i)$ by repeatedly deploying $\theta_i = 0$ for $i = n+1 ,\dots, 2n$, and
computes $\thetahat_n$ by solving another least-squares problem
\begin{align*}
    \thetahat_n \in \argmin_{\theta \in \Theta} \sum_{i=n+1}^{2n} (y_i - \theta^\top (x_i - \estmu))^2.
\end{align*}
\item \textbf{DFO.} We run the derivative-free optimization procedure
from Flaxman et al.~\cite{flaxman2005}. We initialize $\theta_0 = \mathbf{1}$, use 
step-size sequence $c_0 / t$, $c_0 = 0.01$, a batch size of 20 samples per-step,
and take $\delta = 10$.  These parameter were chosen via a small grid search
over $c_0 \in [1e-4, 1]$, batch size in $[1, 500]$, and $\delta \in [0.1, 100]$.
However, the algorithm still has variance across runs, especially in the small
$\eps$ regime.
\item \textbf{Greedy SGD.} We use the greedy SGD variant introduced
by Mendler-D\"{u}nner et al.~\cite{mendler2020stochastic} with initial point $\theta_0 = \mathbf{1}$ and
step-size sequence $1/\sqrt{t}$, which we found to slightly outperform the
step-sequence $1/t$ in our experiments.  
For the sake of brevity, we omit the full pseudocode of greedy/lazy SGD instead point the reader to Figure 1 in \cite{mendler2020stochastic}.
\item \textbf{Lazy SGD.} We use the lazy SGD algorithm \cite{mendler2020stochastic} with initial point $\theta_0 = \mathbf{1}$,
step-size sequence $c / (k_0 + t)$ with parameters $c = 1, k_0 = 1$, and $k^2$ collected samples in $k$-th update.
\end{enumerate}
\paragraph{Evaluation.}
We ran each algorithm for 50 trials, and in \figureref{linear_regression_experiment}, we compare the suboptimality gap
$\PR(\theta) - \PR(\thetaPO)$ of each algorithm as a function of the number of
samples. For each sample size $n$, we bootstrap 95\% confidence
intervals over the 50 trials.

\subsection{Strategic Classification}
\paragraph{Data generation.}
We use the same strategic classification simulator
as~\cite{perdomo2020performative}. For detailed information about the simulator, please refer to Appendix B.2
of~\cite{perdomo2020performative}.

The strategic responses are determined according to
$$x_{\mathrm{BR}} = x + \eps B\theta,$$
for some matrix $B$ which determines the subset of features that are performative. Here,  $\theta\in\R^{11}$ parameterizes a logistic regression classifier. The logistic loss is regularized by an additional $\ell_2$-penalty, which makes it strongly convex. The computation of the smoothness parameter can be found in \cite{perdomo2020performative}.

We consider two different values of the sensitivity parameter, $\epsilon\in\{0.0001,100\}$, and set the magnitude of the regularizer to be $\lambda = 0.002$. We restrict the radius of the optimization domain to be 10, $\Theta = \{\theta:\|\theta\|_2 \leq 10\}$. This choice of parameters ensures that $\epsilon = 0.0001$ is below the critical threshold $\frac{\gamma}{2\beta}$, while $\epsilon = 100$ is above the threshold.

\paragraph{Algorithms.}
We compare the same four algorithms as the previous section.
\begin{enumerate}
\item \textbf{Two-stage procedure.} In the first stage, we deploy random $\theta_i \sim \mathcal{N}(0,I)$ and
perform linear regression to estimate $\mu$,
\[
\hat{\mu} = \argmin_{\mu} \sum_{i=1}^n\norm{z_i - \mu \theta_i}^2
\]
Then, having collected samples from the base distribution, we solve the proxy logistic regression objective offline by running gradient descent with a line search procedure until a tolerance criterion is met. In particular, we solve, 
 \[\argmin_{\theta \in \Theta} \frac{1}{n}\sum_{j=n+1}^{2n} \ell(z_j +
    \estmu\theta; \theta),\]
 where $\ell(z ; \theta)$, is the regularized logistic regression objective, until the improvement between consecutive iterates is smaller than 1e-10. 
\item \textbf{DFO.} We again run the derivative-free optimization procedure
from Flaxman et al.~\cite{flaxman2005}. We initialize $\theta_0 = \mathbf{0}$, use 
step-size sequence $1 / t$, a batch size of 100 samples per-step, and set $\delta = 1$. We tried several other parameter configurations and found this one to perform best on this problem setting.
\item \textbf{Greedy SGD.} We run the greedy SGD variant with initial point $\theta_0 = \mathbf{0}$ and
step-size sequence as suggested by \cite{mendler2020stochastic}. See Appendix A in \cite{mendler2020stochastic} for details.
\item \textbf{Lazy SGD.} We use the lazy SGD algorithm with initial point $\theta_0 = \mathbf{0}$ and $k^2$ collected samples in $k$-th update. As for greedy SGD, we use the step-size sequence suggested by \cite{mendler2020stochastic}.
\end{enumerate}

\paragraph{Evaluation.}
We ran each algorithm for 50 trials, and in \figureref{strat_class_experiment}, we compare the performative risk
$\PR(\theta)$ of each algorithm as a function of the number of
samples. For each sample size $n$, we bootstrap 95\% confidence
intervals over the 50 trials.

\end{document}